\documentclass{article}
\usepackage[preprint]{include/ProbNum25} 

\usepackage{hyperref}       
\usepackage{url}            
\usepackage{amsfonts} 
\usepackage{amsmath} 
\usepackage[dvipsnames]{xcolor}
\usepackage{enumitem}
\usepackage{nicefrac}  
\usepackage{arydshln}     
\usepackage{microtype}      
\usepackage{lipsum}
\usepackage{graphicx}       
\graphicspath{{media/}}     
\usepackage{amssymb}
\usepackage{mathrsfs}
\usepackage{amsthm}
\usepackage{adjustbox}
\usepackage{booktabs}     
\usepackage{tabularx}     
\usepackage{bbm}
\usepackage{pifont}

\usepackage{array}        
\usepackage{threeparttable} 
\usepackage{multirow}     %
\usepackage[nameinlink]{cleveref}
\usepackage{appendix}
\usepackage{caption}
\usepackage{subcaption}
\allowdisplaybreaks

\newtheorem{theorem}{Theorem}
\newtheorem{remark}{Remark}

\definecolor{mydarkblue}{rgb}{0,0.08,0.45}

\hypersetup{ %
    pdftitle={},
    pdfauthor={},
    pdfsubject={},
    pdfkeywords={},
    pdfborder=0 0 0,
    pdfpagemode=UseNone,
    colorlinks=true,
    linkcolor=mydarkblue,
    citecolor=mydarkblue,
    filecolor=mydarkblue,
    urlcolor=mydarkblue,
    pdfview=FitH
}

\usepackage[round]{natbib}

\newcommand{\rneg}{\mathord{\reflectbox{$\neg$}}}
\newcommand{\one}{\mathbbm{1}}

\newenvironment{talign*}
 {\csname align*\endcsname}
 {\endalign}
\newenvironment{talign}
 {\csname align\endcsname}
 {\endalign}
 
\probnumtitle{BayesSum: Bayesian Quadrature in Discrete Spaces}
\probnumauthors{%
\name{Sophia Seulkee Kang}%
\affiliation{Independent Researcher}
\and%
\name{Fran\c{c}ois-Xavier Briol}%
\affiliation{University College London}
\and%
\name{Toni Karvonen}%
\affiliation{Lappeenranta--Lahti University of Technology LUT}
\and%
\name{Zonghao Chen}%
\affiliation{University College London}
}

\probnumabstract{
This paper addresses the challenging computational problem of estimating intractable expectations over discrete domains. 
Existing approaches, including Monte Carlo and Russian Roulette estimators, are consistent but often require a large number of samples to achieve accurate results.
We propose a novel estimator, \emph{BayesSum}, which is an extension of Bayesian quadrature to discrete domains. It is more sample efficient than alternatives due to its ability to make use of prior information about the integrand through a Gaussian process. We show this through theory, deriving a convergence rate significantly faster than Monte Carlo in a broad range of settings.
We also demonstrate empirically that our proposed method does indeed require fewer samples on several synthetic settings as well as for parameter estimation for Conway-Maxwell-Poisson and Potts models. 
}

\begin{document}
\newcommand{\E}{\mathbb{E}}
\newcommand{\R}{\mathbb{R}}
\newcommand{\Qb}{\mathbb{Q}}
\newcommand{\Pb}{\mathbb{P}}
\newcommand{\calH}{\mathcal{H}}
\newcommand{\calN}{\mathcal{N}}
\newcommand{\calM}{\mathcal{M}}
\newcommand{\calE}{\mathcal{E}}
\newcommand{\calT}{\mathcal{T}}
\newcommand{\calL}{\mathcal{L}}
\newcommand{\calS}{\mathcal{S}}
\newcommand{\calF}{\mathcal{F}}
\newcommand{\calG}{\mathcal{G}}
\newcommand{\calZ}{\mathcal{Z}}
\newcommand{\calB}{\mathcal{B}}
\newcommand{\calY}{\mathcal{Y}}
\newcommand{\calO}{\mathcal{O}}
\newcommand{\calD}{\mathcal{D}}
\newcommand{\calX}{\mathcal{X}}
\newcommand{\calI}{\mathcal{I}}
\newcommand{\calP}{\mathcal{P}}
\newcommand{\dd}{\mathrm{d}}
\newcommand{\bx}{\mathbf{x}}
\newcommand{\bo}{\mathbf{o}}
\newcommand{\by}{\mathbf{y}}
\newcommand{\bz}{\mathbf{z}}
\newcommand{\bH}{\mathbf{H}}
\newcommand{\bw}{\mathbf{w}}
\newcommand{\bJ}{\mathbf{J}}
\newcommand{\bnabla}{\boldsymbol{\nabla}}
\newcommand{\scrG}{\mathscr{G}}
\newcommand{\scrF}{\mathscr{F}}
\newcommand{\scrX}{\mathscr{X}}
\newcommand{\scrx}{\mathscr{X}}
\newcommand{\scrY}{\mathscr{Y}}
\newcommand{\scry}{\mathscr{Y}}
\newcommand{\scrW}{\mathscr{W}}
\newcommand{\scrw}{\mathscr{W}}
\newcommand{\scrZ}{\mathscr{Z}}
\newcommand{\scrz}{\mathscr{Z}}
\newcommand{\scrU}{\mathscr{U}}
\newcommand{\scrR}{\mathscr{R}}
\newcommand{\scrQ}{\mathscr{Q}}
\newcommand{\scrL}{\mathscr{L}}

\newcommand{\Id}{\mathrm{Id}}
\newcommand{\N}{\mathbb{N}}
\newcommand{\Z}{\mathbb{Z}}
\newcommand{\kl}{\mathrm{KL}}
\newcommand{\fisher}{\mathrm{FI}}
\newcommand{\tv}{\mathrm{TV}}
\newcommand{\LSI}{\mathrm{LSI}}

\newcommand{\kq}{\mathrm{KQ}}
\newcommand{\bq}{\mathrm{BQ}}
\newcommand{\mc}{\mathrm{MC}}

\section{Introduction}
We consider the challenge of computing an expectation over a discrete domain, which can typically be expressed as an intractable sum to be estimated. 
More precisely, for a discrete domain $\mathcal{X}$, an integrable function $f: \mathcal{X} \to \mathbb{R}$, and a probability measure $\mathbb{P}$ on $\mathcal{X}$ with probability mass function $p:\mathcal{X}\to\mathbb{R}$, the quantity of interest is: 
\begin{align}\label{eq:I}
    I = \E_{X\sim \Pb}[f(X)] = \sum_{x \in \mathcal{X}} f(x) p(x).
\end{align}
In practice, $I$ is typically intractable either because there are countably many terms (e.g., when $\mathcal{X} = \N$), or because there is a finite but prohibitively large number of terms, such as the set of all configurations in a combinatorial structure such as a lattice. 

The challenge of estimating $I$ arises for statistical and machine learning models of discrete data which \emph{lack} a closed-form normalization constant, rendering maximum likelihood training, posterior inference, and model evaluation computationally intractable. Prominent examples of unnormalized likelihood models include the Ising/Potts model for protein sequences~\citep{balakrishnan2011learning} and statistical physics~\citep{baxter2016exactly}, spatial point processes for ecological and epidemiological data~\citep{moller2003statistical}, multivariate count data~\citep{piancastelli2023multivariate}, and exponential random graph models for social networks~\citep{robins2007recent}. 
In Bayesian inference, the intractable constant of these likelihood models (e.g., a Potts model with large state spaces) along with the intractable posterior normalizing constant leads to the notorious \emph{doubly intractable} posterior distributions~\citep{murray2006mcmc}.

Despite its broad relevance, surprisingly few methods exist in the literature for estimating $I$ (in \eqref{eq:I}) as compared to its counterpart in the continuous domain.
The most straightforward approach is the Monte Carlo (MC) method~\citep{rubinstein2016simulation}, where $N$ i.i.d. samples $x_1, \ldots, x_N$ are drawn from $\Pb$ and the intractable sum is approximated by the empirical average $\hat{I}_{\mc}=\frac{1}{N}\sum_{i=1}^N f(x_i)$. 
Under mild conditions, $\hat{I}_\mc$ is consistent, however, its convergence rate remains relatively slow, as MC does not exploit any structural properties of the integrand $f$ such as smoothness or sparsity.
In the continuous domain, a family of approaches have been proposed that leverage such properties of $f$ and consequently results in faster convergence rate. Prominent examples include quasi Monte Carlo~\citep{dick2011higher}, 
Stein variational gradient descent~\citep{liu2018stein}, 
classical cubature methods~\citep{davis2007methods}, kernel herding and Stein points \citep{Chen2010,Chen2018}
kernel thinning~\citep{dwivedi2024kernel}, stationary points~\citep{chen2025stationary},  and Bayesian quadrature~\citep{briol2019probabilistic}. 
In contrast, analogous developments for the discrete domain are far more limited.

In this paper, we propose a novel approach for discrete domains called \emph{BayesSum}. 
The name comes from the fact that our approach extends the Bayesian quadrature algorithm~\citep{o1991bayes,Diaconis1988,rasmussen2003bayesian,briol2019probabilistic} to the computation of intractable expectations taking the form of sums over a discrete domain $\calX$. 
By leveraging the structural properties of the integrand $f$ through a prior Gaussian process~\citep{williams2006gaussian} over $\calX$, 
BayesSum achieves a fast rate of convergence, i.e., we observe in our experiments that it requires fewer samples than competing methods to reach the same approximation accuracy. 
BayesSum falls within the framework of Bayesian probabilistic numerical methods~\citep{cockayne2019bayesian,hennig2022probabilistic} in that it provides finite-sample Bayesian uncertainty quantification.

In addition to the basic BayesSum algorithm, we introduce the following variants that further improve its applicability and efficiency: 
\begin{enumerate}[itemsep=2.0pt, topsep=2pt, leftmargin=*]
    \item \emph{BayesSum + Bayesian quadrature}. BayesSum can be seamlessly combined with Bayesian quadrature, resulting in a novel algorithm for estimating intractable expectations over mixed domains with both discrete and continuous components. Computing expectations over mixed domains remains largely underexplored in the literature, with only a few baselines available.
    \item \emph{Active BayesSum}. We propose active BayesSum, which selects future query locations by maximizing a utility function, such as expected information gain. This active learning strategy further enhances the sample efficiency of BayesSum by only performing function evaluations at locations most useful for computing the expectation. 
    \item \emph{Stein BayesSum}. Similarly to Bayesian quadrature, a central ingredient of BayesSum is the availability of closed-form kernel mean embeddings. To this end, we provide in \Cref{tab:discrete_kme} a comprehensive dictionary of kernel–distribution pairs that admit such closed-form expressions in the discrete setting. For cases where closed-form embeddings are not available (e.g., where $\Pb$ is known only up to a normalization constant), we propose Stein BayesSum, which leverages discrete Stein kernels~\citep{yang2018goodness} to bypass the need for explicit kernel mean embeddings.
\end{enumerate}
We empirically verify the effectiveness of all of these variants of BayesSum over several synthetic examples where we have access to the ground truth of \eqref{eq:I} which allows extensive benchmarking. We then also test BayesSum in more realistic settings, including to perform parameter estimation for a Conway–Maxwell–Poisson model of count data, and for a large Potts model on a set of sequences resembling biological proteins.

\section{Background and Related Work}\label{sec:related_work}

In this section, we begin by reviewing existing methods for estimating intractable expectations over discrete domains. 
Notably, none of these methods exploit structural information about the integrand $f$. 
Next, we present Bayesian quadrature, a widely used framework for numerical integration in continuous spaces. Finally, we introduce Gaussian processes and reproducing kernels on discrete domains $\mathcal{X}$, which provide the technical foundations for BayesSum.

\subsection{Existing methods}

\textbf{Importance sampling:} Beyond the standard Monte Carlo estimator introduced above, alternative sampling distributions are often helpful.
Importance sampling estimates the expectation in \eqref{eq:I} by rewriting the sum under an easy-to-sample proposal distribution $\Qb$ with mass function $q$ that shares the same support as $\Pb$. Specifically, $I=\sum_{x \in\calX} f(x) \frac{p(x)}{q(x)} q(x) = \mathbb{E}_{X \sim \Qb}[f(X) \frac{p(X)}{q(X)}]$. 
This formulation allows us to approximate $I$ by drawing samples from $\Qb$ and averaging the alternative integrand $f(x) p(x) / q(x)$~\citep{mcbook}. 
The choice of $\Qb$ is critical to the variance of the final estimator.

\textbf{Russian roulette:} 
A specific approach to estimate ~\eqref{eq:I} over a discrete domain without exhaustively summing over all configurations in $\calX$ is to use an unbiased random truncation scheme, often called the \emph{Russian roulette} estimator. The key idea is to evaluate the summand only on a randomly selected subset $\calX'\subset \calX$, and to compensate for the missing terms through appropriately re-weighting so that the final estimator remains unbiased. 
In practice, the Russian roulette estimator is known to suffer from high variance~\citep{lyne2015russian,hendricks2006mcnp}. Although it bears similarity to importance sampling, it differs fundamentally by summing over only a random subset $\calX'$ instead of the entire $\calX$. 

\textbf{Stratified sampling:} 
Another approach to estimating \eqref{eq:I} is stratified sampling, where the domain is partitioned into $M$ disjoint subsets and $\lfloor N/M \rfloor$ samples are drawn from the conditional distribution within each subset, thereby enforcing more uniform coverage of the domain~\citep[Chapter 10]{mcbook}. In continuous settings, quasi–Monte Carlo (QMC) methods can be viewed as a refinement of stratified sampling, using increasingly finer partitions that yield low-discrepancy point sets~\citep{caflisch1998monte,owen2003quasi}. However, QMC methods are not directly applicable to discrete domains, as it is unclear how to construct low-discrepancy point sets in general discrete spaces. 
Another limitation of stratified sampling is that it requires a tractable partition of $\calX$ and efficient sampling from each partition, which becomes challenging in domains of high dimensionality. 

\subsection{Bayesian quadrature} 
The task of estimating $I$ in the continuous setting, i.e., $I=\int f(x) p(x) \mathrm{d} x$, using weighted averages of function evaluations is known as quadrature~\citep{gautschi1981survey}. Quadrature has a long history, and numerous sophisticated methods have been developed~\citep{novak2006deterministic, dick2011higher,bardenet2020monte,sun2023vector,chopin2024higher,chen2025stationary}.
Of particular relevance to our work is \emph{Bayesian quadrature} (BQ)~\citep{o1991bayes,Diaconis1988,rasmussen2003bayesian,briol2019probabilistic}, in which the estimator is obtained by integrating the posterior mean of a Gaussian process (GP). 
It has been shown both theoretically and empirically, in the continuous domain, that BQ estimators attain a faster convergence rate than standard Monte Carlo~\citep{Kanagawa2016,briol2019probabilistic,Kanagawa2019,Karvonen2020,Wynne2020,chennested}; in other words, for a given error tolerance, BQ requires significantly fewer samples.
This improvement in sample efficiency arises from exploiting the smoothness of the integrand $f$ encoded through the GP prior covariance. 
Remarkably, BQ continues to achieve superior performance even when the integrand $f$ is less smooth than samples drawn from the GP prior, i.e., in the misspecified setting~\citep{kanagawa2020convergence}.
BQ has been used, for instance, in estimation of normalization constants~\citep{osborne2012active}, in reinforcement learning~\citep{paul2018alternating} and simulation-based inference~\citep{Bharti2023}. It has also been extended to conditional and nested expectations~\citep{chen2024conditional,chennested} as well as multi-output~\citep{xi2018bayesian} and multi-fidelity settings~\citep{li2023multilevel}.

\subsection{Gaussian Process and Reproducing kernel Hilbert space on discrete domains}\label{sec:gp_rkhs}

Here we briefly recall Gaussian processes and kernels defined on discrete domains, which we will use in BayesSum to encode prior knowledge about the integrand $f$. 

For a positive semi-definite kernel $k: \calX \times \calX \rightarrow \R$, its unique associated reproducing kernel Hilbert space (RKHS) $\calH_k$ is a Hilbert space with inner product $\langle\cdot, \cdot\rangle_{\calH_k}$ and norm $\lVert\cdot\rVert_{\calH_k}$~\citep{aronszajn1950theory} such that: (i) $k(x, \cdot) \in \calH_k$ for all $x \in \calX$, and (ii) the reproducing property holds, i.e. for all $h \in \calH_k, x \in \calX$, $h(x)=\langle h, k(x, \cdot)\rangle_{\calH_k}$.

Although kernels defined on continuous domains, e.g., Gaussian RBF kernels, remain valid (i.e., symmetric and positive semi-definite) when restricted to discrete domains~\citep[Theorem 4.16]{steinwart2008support}, it is unclear whether the corresponding RKHS retain the same structural (e.g smoothness) properties as in the continuous case. 
For example, on an open subset $\mathcal{X}\subseteq \mathbb{R}^d$, the RKHS of a Matérn kernel on $\mathcal{X}\subseteq \mathbb{R}^d$ coincides with a Sobolev space, consisting of functions with square-integrable derivatives up to a certain order~\citep[Corollary 10.48]{wendland2004scattered}. 
However, when the domain $\mathcal{X}$ is discrete, it is unclear whether such smoothness remain true. 
An exception is the Brownian motion kernel $k(x,y) = \text{min}(x,y)$ which gives the minimum of the two arguments. 
On the discrete domain $\mathcal{X}=\N$, its RKHS $\mathcal{H}_k$ consists of functions $f:\N\to\mathbb{R}$ satisfying $\sum_{n\in\N} |f(n+1)-f(n)|^2 < \infty$~\citep{jorgensen2015discrete}, similar to its continuous analogue. 

Fortunately, there exist several reproducing kernels which are  explicitly designed for discrete domains. 
The most widely used among them is the
\emph{Hamming distance kernel}: for $\calX = \{0,1\}^L$, the kernel is $k(x,y) = d_H(x,y)$ where $d_H(x,y) =\sum_{i=1}^L 1\{x_i \neq y_i\}$ represents the Hamming distance between two sequences of length $L$. 
This kernel encodes the prior knowledge that sequences are more similar if they differ in fewer positions, making it useful for biological sequences. 
Another alternative is the \emph{exponential Hamming kernel} proposed in \citet{amin2023biological}  where
$k(x,y) = A \exp(-\lambda d_H(x,y))$ with $A$ and $\lambda$ being the kernel amplitude and lengthscale. 
This kernel encodes similar prior knowledge as above and it achieves better empirical performances in kernel ridge regression~\citep{amin2023biological}, making it the default kernel in \texttt{botorch} for categorical inputs~\citep{balandat2020botorch}. 
There also exist other kernels defined over discrete domains which are tailored towards specific applications; see \citet{gartner2008kernels} and \citet{jorgensen2015discrete} for surveys.

Apart from RKHSs, another popular framework built upon positive semi-definite kernels is Gaussian processes (GPs)~\citep{williams2006gaussian}, denoted as $\mathcal{GP}(m, k)$, which is a stochastic process over functions $f: \mathcal{X} \rightarrow \mathbb{R}$ such that every finite collection $[f(x_1), \ldots, f(x_n)]^\top$ follows a multivariate normal distribution with mean $[m(x_1), \ldots, m(x_n))]^\top$ and covariance $[k(x_i, x_j)]_{i, j} \in \R^{n\times n}$. 
One particular application of Gaussian processes over discrete domains $\calX$~\citep{fortuin2021sparse} is their use as probabilistic surrogate models in Bayesian optimization when the black-box function to be optimized is defined on the discrete domains~\citep{balandat2020botorch,ru2020bayesian,daulton2022bayesian}.

\section{Methodology}\label{sec:dbq}
\vspace{-10pt}
We are now ready to introduce  \emph{BayesSum}, a Bayesian estimator for the intractable expectation in Eq.~\eqref{eq:I}, 
with $N$ samples $x_{1:N} := [x_1, \ldots, x_N]^\top$ and corresponding function evaluations $f(x_{1:N}) := [f(x_1), \ldots, f(x_N)]^\top$.  
\subsection{BayesSum}

We begin by positing a GP prior on $f$. We denote this prior $\mathcal{GP}(m, k)$, with $m: \calX \rightarrow \mathbb{R}$ a mean function and $k: \calX \times \calX \rightarrow \mathbb{R}$ a positive semi-definite kernel function as reviewed above. 
Without loss of generality, here we take $m \equiv 0$ as any nonzero mean can be absorbed by centering the observations (i.e., subtracting the empirical mean). The kernel function 
$k$ encodes prior knowledge about the notion of similarity or distance on 
$\calX$ (e.g., the Hamming distance for biological strings).

Once a GP prior has been selected and conditioned on \emph{noiseless} function values $f(x_{1: N})=[f(x_1), \cdots, f(x_N)]^{\top}$, we obtain a posterior GP on $f$ which induces a univariate Gaussian posterior distribution $\mathcal{N}(\hat{I}, \hat{\sigma}^2)$ on $I$ with:
\begin{align}\label{eq:BQ}
\hat{I} & := \mu_\Pb(x_{1: N})^{\top} \mathbf{K}^{-1} f(x_{1: N}), \tag{BayesSum} \\
\hat{\sigma}^2 & := \E_{X, X^{\prime} \sim \mathbb{P}}[k(X, X^{\prime})] -\mu_\Pb(x_{1: N})^{\top} \mathbf{K}^{-1} \mu_\Pb(x_{1: N}). \nonumber 
\end{align}
Here, $\mu_\mathbb{P}(x) = \E_{X\sim\Pb}[k(X, x)]$ is the kernel mean embedding, $\E_{X, X^{\prime} \sim \mathbb{P}}[k(X, X^{\prime})]$ is often called the initial error, and $\mathbf{K} = k(x_{1:N}, x_{1:N}) \in\R^{N\times N}$ is the associated Gram matrix with entries $\mathbf{K}_{i,j} = k(x_i, x_j)$. 
Beyond the posterior mean, BayesSum also provides a posterior variance only at a marginal additional cost, which quantifies the epistemic uncertainty due to the reliance on a finite set of function evaluations---a perspective rooted in the framework of probabilistic numerics~\citep{hennig2022probabilistic}. 
We refer to the proposed estimator as \emph{BayesSum}\footnote{Despite the similar terminology, this shall not be confused with kernel Bayes rule~\citep{fukumizu2011kernel} or kernel sum rule~\citep{nishiyama2020model}.}. 

\begin{remark} \label{rmk:weights}
    The BayesSum posterior mean and variance mirror that of Bayesian quadrature \emph{exactly}. The only distinction is that our domain $\calX$ is discrete rather than continuous. 
    For this reason, referring to the method as quadrature would be misleading, since quadrature concerns numerical integration over continuous domains. 
    Given this equivalence, the BayesSum posterior mean can be written as a weighted sum of function evaluations $\sum_{i=1}^N w_i f(x_i)$, where the weights $w_{1:N} = \mu_\Pb(x_{1: N})^{\top} \mathbf{K}^{-1}$ are \emph{optimal} in the associated RKHS~\citep{huszar2012optimally,briol2019probabilistic}.
\end{remark}

\textbf{Computational cost:} 
The computational cost of our BayesSum with $N$ samples is $\calO(N^3)$, which is much larger than the $\calO(N)$ cost of Monte Carlo. Numerous approaches have been proposed in the literature to reduce the cost of Bayesian quadrature in the continuous domain: for instance, one can use geometric properties of the point set~\citep{karvonen2018fully, karvonen2019symmetry,kuo2025constructing}, Nyström approximations~\citep{hayakawa2022positively,hayakawa2023sampling}, or randomly pivoted Cholesky~\citep{epperly2023kernel}.
Among these methods, Nyström approximations and randomly pivoted Cholesky can also be applied to the discrete domain. 
Moreover, since the BayesSum weights do not depend on the integrand $f$, they can be precomputed once and reused, reducing the computational cost to 
$\calO(N)$~\citep[Appendix F.1]{chennested}. 
See \Cref{sec:experiment} for details on how this is implemented in our experiments.
Furthermore, in scenarios where the dominant cost arises from function evaluations (e.g., expensive computer simulations) or from obtaining samples (e.g., Markov chain Monte Carlo), the superior sample efficiency of BayesSum allows it to achieve accurate estimates with much fewer evaluations than MC, thereby reducing the overall computational cost in practice. 

\textbf{Hyperparameter selection:}
The choice of kernel hyperparameters, such as amplitude $A$ and lengthscale $\lambda$ for the exponential Hamming distance kernel $k(x,y)=A\exp(-\lambda d_H(x,y))$, is crucial for both the accuracy of the BayesSum estimator and the calibration of its posterior variance.
In our experiments, these are selected via  empirical Bayes, which involves choosing the values that maximize the log marginal likelihood~\citep{williams2006gaussian} from a predefined candidate set, e.g., $\{1.0, 10.0, 100.0, 1000.0\}$ for the amplitude $A$ and $\{0.1, 0.3, 1.0, 3.0, 10.0\}$ for the lengthscale $\lambda$.  

\textbf{Convergence guarantees:}
The following theorem shows that BayesSum estimator $\hat{I}$ converges to the ground truth $I$ if $f \in \mathcal{H}_k$. 
To optimize the rate of convergence one should use samples at which most of the probability mass is concentrated.

\begin{theorem}\label{thm:main}
    Suppose that $k(x, y) \leq C^2$ for all $x, y \in \mathcal{X}$. 
    Suppose that the $N$ samples $\{x_i\}_{i=1}^N$ are non-repetitive. 
    Without loss of generality, let $\mathcal{X}$ admit an enumeration $\mathcal{X} = \{x_i\}_{i=1}^\infty$ where the first $N$ samples $\{x_i\}_{i=1}^N$ correspond to the observed samples. 
    If $f \in \mathcal{H}_k$, then
    \begin{talign*}
        \lvert I - \hat{I} \rvert \leq C \lVert f \rVert_{\mathcal{H}_k} \big( 1 - \sum_{i=1}^N p(x_i) \big) \to 0 \: \text{ as } \: N \to \infty.
    \end{talign*}
\end{theorem}
\begin{proof}
BayesSum is a special case of Bayesian quadrature, which is worst-case optimal in $\mathcal{H}_k$.
From this it follows that, for any weights $w_1, \ldots, w_N \in \mathbb{R}$, 
\begin{talign} \label{eq:bq-bound-w}
    \hat{\sigma}^2 \leq{}& \E_{X, X^{\prime} \sim \mathbb{P}}[k(X, X^{\prime})] - 2 \sum_{i=1}^N w_i \mu_{\mathbb{P}}(x_i) \nonumber \\
    &+ \sum_{i,j=1}^N w_i w_j k(x_i, x_j).
\end{talign}
We refer to Section~2 in \cite{briol2019probabilistic} for this result.
Select $w_i = p(x_i)$.
Then it is easy to compute that the right-hand side in~\eqref{eq:bq-bound-w} equals $\sum_{i,j > N}k(x_i, x_j) p(x_i) p(x_j)$.
Equation~(5) in~\cite{briol2019probabilistic} gives $\lvert I - \hat{I} \rvert \leq \lVert f \rVert_{\mathcal{H}_k} \hat{\sigma}$.
The claim then follows from $k(x, y) \leq C^2$ and the identity $\sum_{i,j > N} p(x_i) p(x_j) = (1 - \sum_{i=1}^N p(x_i))^2$.
\end{proof}
Note that the above proof also holds for finite $\calX$.
The example in \Cref{sec:example} shows that the upper bound in Theorem~\ref{thm:main} is relatively tight. It also shows that $\lvert I - \hat{I} \rvert = \calO(N^{-\alpha + 1})$ for a function $f$ of smoothness $\alpha$, which improves upon the standard Monte Carlo rate $\mathcal{O}(N^{-1/2})$ whenever $\alpha > 3/2$. 

\begin{remark}[Non-repetitive points]\label{rem:iid_repeat}
We would like to highlight that our theory requires non-repetitive  samples. This is a distinctive feature in the discrete domain, since in continuous domains the probability of drawing identical samples is zero. 
We observe empirically that repetitive sampling leads to notably
larger errors in the experiments (see \Cref{fig:iid_repetitive}).  We attribute this degradation in performance to two factors: (i) the resulting Gram matrix becomes singular, and (ii) conditioning the GP prior on repetitive samples leaves the BayesSum posterior unchanged, providing no additional information. 
\end{remark}
\vspace{-10pt}

\begin{table*}[ht]
\centering
\renewcommand{\arraystretch}{1.25}
\small
\begin{adjustbox}{max width=\textwidth} 
\begin{threeparttable}
\caption{Closed-form kernel mean embeddings and initial errors. 
The exact formulas of initial error can be found in \Cref{sec:derivation}. 
$\vartheta$ denotes Jacobi theta function, $Q$ denotes regularized Gamma function~\citep{silverman1972special}, $I_p$ denotes regularized incomplete Beta function~\citep{silverman1972special}, $T_n$ denotes Touchard polynomial function~\citep{touchard1939cycles}, $M_n$ denotes the generalized Touchard polynomial function, and $B_n$ denotes complete exponential Bell polynomial function~\citep{bell1934exponential}.
We omit kernel amplitude $A$, as it merely rescales the kernel mean embeddings. Detailed derivations are included in \Cref{sec:derivation}. 
} 
\label{tab:discrete_kme}
\vspace{-15pt}
\begin{tabularx}{\textwidth}{@{} >{\centering\arraybackslash}l
>{\centering\arraybackslash}l
>{\centering\arraybackslash}l
>{\centering\arraybackslash}l
>{\centering\arraybackslash}X @{}}
\toprule
$\calX$ & $\Pb$ & Kernel & Kernel mean embedding $\mu_\Pb(y)$ & Initial error \\
\midrule 
$\N_0$ & Poisson($\eta$) & $k(x,y)=x\wedge y$
& $\sum_{n=0}^{y} n \,\frac{e^{-\eta}\eta^n}{n!} + y\!\left(1 - Q(y{+}1,\eta)\right)$ & \ding{51} \\ \addlinespace

$\N_0$ & NB($\tau,q$) & $k(x,y)=x\wedge y$
& $y - \sum_{n=0}^{y-1} I_q(\tau,n{+}1)$ & \ding{51} \\ \addlinespace

$\N_+$ & Log($p$) & $k(x,y)=x\wedge y$
& $y + \frac{1}{\ln(1-p)} \sum_{n=1}^{y-1}\frac{(y-n)p^n}{n}$ & \ding{55} \\ \addlinespace

$\N_0$ & Poisson($\eta$) & $k(x,y)=(xy + 1)^r$
& $\sum_{n=0}^r \binom{r}{n} y^n T_n(\eta)$ & \ding{51} \\ \addlinespace


$\N_0$ & NB($\tau,q$) & $k(x,y)=(xy + 1)^r$
& $\sum_{n=0}^r \binom{r}{n} y^n M_n(r,q)$ & \ding{51} \\ \addlinespace

$\Z$ & Skellam($\lambda_1,\lambda_2$) & $k(x,y) = (xy+1)^r$
& $\sum_{n=0}^r \binom{r}{n} y^n B_n(\lambda_1 - \lambda_2, \ldots, \lambda_1 + (-1)^n \lambda_2)$ & \ding{51} \\ \addlinespace

$\{-1,+1\}^d$ & Uniform & $k(x,y)=\exp(-\lambda d_H(x,y))$ 
& $ \exp(-0.5 \lambda d)
\left[\cosh\!\left(\tfrac{\lambda}{2}\right)\right]^d$ & \ding{51} \\ \addlinespace

$\{0, \ldots, m\}^d$ & Uniform & $k(x,y)=\exp(-\lambda d_H(x,y))$ 
& $ \Big(\frac{1 + m e^{-\lambda}}{m+1}\Big)^d$ & \ding{51} \\ \addlinespace

$\{0,1\}^d$ & Uniform & $k(x, y) = \frac{x^\top y}{\|x\|^2+\|y\|^2-x^\top y}$ & 
Equation \eqref{eq:kme_tanimoto}
& \ding{51} \\
\bottomrule
\end{tabularx}
\end{threeparttable}
\end{adjustbox}
\vspace{-10pt}
\end{table*}

\subsection{Estimate intractable expectations over mixed domains}\label{sec:mixed_bq}
BayesSum can be seamlessly combined with Bayesian quadrature (BQ), resulting in a novel estimator for intractable expectations over \emph{mixed domains}, i.e., domains that involve both continuous and discrete parameters. 
Consider a function $f:\calX\times\calY\to\R$ which takes in two inputs: $x\in\calX$ is a discrete input (e.g. natural numbers) and $y\in\calY$ is a continuous input (e.g. real vectors). Let $\Pb$ be a probability measure over this mixed domain and the target of interest in this setting is an intractable expectation over the mixed domain: 
\begin{align}\label{eq:I_mixed}
    I = \E_{(X,Y)\sim\Pb}[f(X,Y)] . 
\end{align}
Following the same framework as BayesSum, we place a zero-mean Gaussian process prior $\mathcal{GP}(0, k)$ on $f$ with a kernel function $k:(\calX\times\calY) \times (\calX\times\calY)\to\R$. 
One valid option to construct a kernel $k$ over the mixed domain is to take the product of two kernels $k_\calX: \calX\times\calX\to\R$ and $k_\calY: \calY\times\calY\to\R$. 
Conditioned on $N$ noiseless evaluations
$f(x_{1:N}, y_{1:N}) = [f(x_1, y_1), \ldots, f(x_N, y_N)]^{\top}$, 
the GP posterior induces a univariate Gaussian posterior on $I$ with the following mean and variance: 
\begin{align}\label{eq:BQ_mixed}
\hat{I} & := \mu_\Pb(x_{1: N}, y_{1:N})^{\top} \mathbf{K}^{-1} f(x_{1: N}, y_{1:N}), \\
\hat{\sigma}^2 & := \E_{(X,Y), (X^{\prime}, Y^\prime) \sim \Pb}[k_\calX(X, X^{\prime})k_\calY(Y, Y^{\prime})] \nonumber \\
&\quad - \mu_\Pb(x_{1: N}, y_{1:N})^{\top} \mathbf{K}^{-1} \mu_\Pb(x_{1: N}, y_{1:N}).
\end{align}
Here, $\mu_\Pb(x_{1: N}, y_{1:N}) = \E_{\Pb}[k_\calX(X, x_{1: N}) k_\calY(Y, y_{1:N}))]$ is the kernel mean embedding and $\mathbf{K} = k_\calX(x_{1: N}, x_{1: N}) \odot k_\calY(y_{1: N}, y_{1: N})$ where $\odot$ denotes the element-wise product between matrices. 
Since the kernel mean embedding $\mu_\Pb(x,y) $ can be factorized as either $\mu_{\Pb_{Y\mid X=x}}(y) \mu_{\Pb_X}(x) $ or $ \mu_{\Pb_{X\mid Y=y}}(x) \mu_{\Pb_Y}(y)$, it admits a closed-form expression when either is available.

\subsection{Closed-form kernel mean embeddings}\label{sec:tractable_kme}
Despite the claimed advantages, a key limitation of BayesSum is the need for closed-form kernel mean embeddings $\mu_\Pb$ and initial error. 
While a dictionary of such embeddings and initial error is available in the continuous domains~\citep{briol2025dictionary}, we now provide a complementary (non-exhaustive) dictionary for pairs of kernels and distributions over the discrete domain in \Cref{tab:discrete_kme}. 
Beyond the distributions listed in \Cref{tab:discrete_kme}, closed-form embeddings and initial error can also be derived when $\Pb$ can be expressed as the push-forward of a distribution using the standard `change-of-variables' technique~\citep{chennested}. 
In addition, one can also construct kernels with closed-form embeddings for a wider class of distribution $\Pb$—including those that are known up to its normalizing constant—via the \emph{Stein reproducing kernel}~\citep{yang2018goodness}. We call BayesSum with Stein reproducing kernels \emph{Stein BayesSum}. 

Let $\calX$ be a finite discrete domain.
Given a base kernel $k:\calX \times\calX\to \R$, the \emph{Stein} reproducing kernel associated with the distribution $p$ can be constructed via a discrete Stein operator~\citep{shi2022gradient}. 
We focus on a particular discrete Stein kernel based on the difference Stein operators~\citep{yang2018goodness}: 
\begin{align}\label{eq:stein_k}
k_p(x, x^{\prime}) &= s_p(x)^{\top} k(x, x^{\prime}) s_p(x^{\prime}) - s_p(x)^{\top} \Delta_{x^{\prime}}^* k(x, x^{\prime}) \nonumber \\
&\hspace{-1cm} -\Delta_{x}^* k(x, x^{\prime})^{\top} s_p(x^{\prime}) + \operatorname{trace}\left(\Delta_{x, x^{\prime}}^* k(x, x^{\prime})\right).
\end{align}
Here, $\Delta$ and $\Delta^\ast$ are partial difference operators defined with respect to a cyclic permutation and an inverse cyclic permutation, and $s_p(x) = \frac{\Delta_{x} p(x)}{p(x)}$ is the \emph{difference score} function associated with the distribution $p$. 
Notably, $s_p$ is accessible even when $p$ is only known up to its normalization constant. 
See \Cref{sec:details} for detailed definitions.
The Stein kernel $k_p$ is a positive semi-definite reproducing kernel from $\calX \times\calX$ to $\R$.

A nice property of the Stein kernel by its construction is that the kernel mean embedding $\mu_\Pb(x) = 0$ for any $x\in\calX$. However, this means our GP prior on $f$ encodes beliefs that the intractable sum in \eqref{eq:I} has mean zero when the GP prior mean $m\equiv 0$. 
To weaken this, we place a hierarchical flat prior on the GP prior mean~\citep{karvonen2018bayes}, which 
results in the following estimator, referred to as \emph{Stein BayesSum}. 
\begin{align}
\hat{I} & := (\one^\top \mathbf{K}^{-1} \one)^{-1} \one^\top \mathbf{K}^{-1} f(x_{1: N}), \tag{Stein BayesSum} \\
\hat{\sigma}^2 & := (\one^\top \mathbf{K}^{-1} \one)^{-1} . \nonumber 
\end{align}
Here, $\mathbf{K} = k_p(x_{1:N}, x_{1:N}) \in\R^{N\times N}$ is the associated Gram matrix of the Stein kernel and $\one \in\R^N$ is a vector with all entries taking value one. 
A similar form of the estimator for numerical integration in the continuous domain has been adopted in \citet[Section 2.3]{karvonen2018bayes} and \citet[Lemma 3]{oates2017control}. 
Previous Bayesian quadrature methods based on Stein kernels, e.g. ~\citet[Appendix B.2]{chen2024conditional}, introduce an additional additive constant $c$ into the kernel. This constant is treated as a kernel hyperparameter and optimized by maximizing the log-marginal likelihood, which results in a higher computational cost compared to our estimator. 

\begin{figure*}[t]
    \begin{minipage}{0.99\textwidth}
    \centering
    \includegraphics[width=360pt]{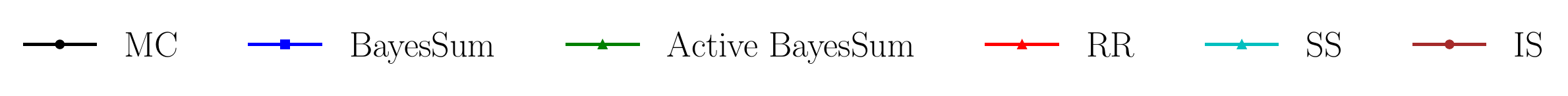}
    \vspace{-8pt}
    \end{minipage}
    \centering
    \vspace{-5pt}
    \begin{subfigure}{0.33\textwidth}
        \centering
        \hspace{-10pt}
        \includegraphics[width=1.0\linewidth]{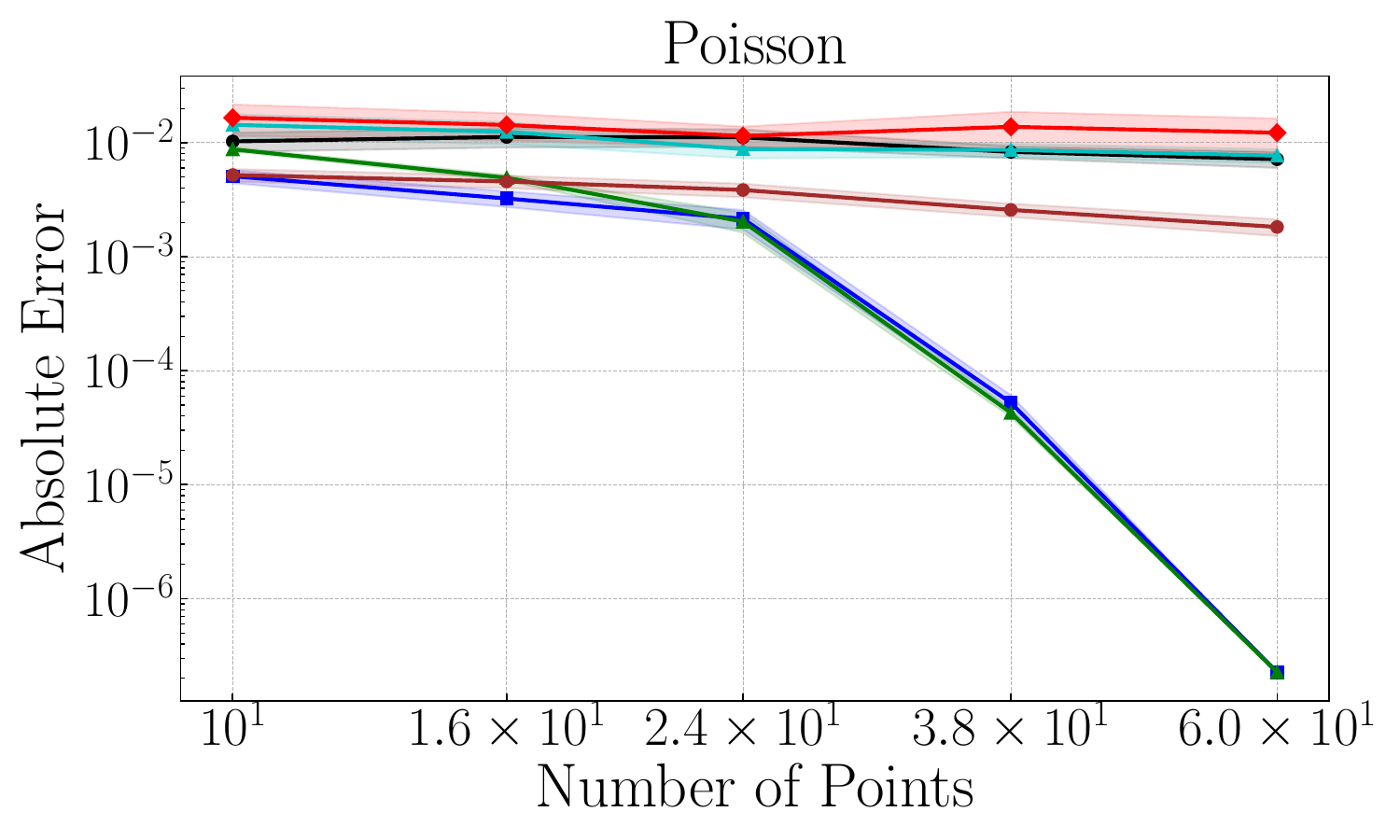}
        \label{fig:poisson}
    \end{subfigure}%
    \hfill 
    \begin{subfigure}{0.33\textwidth}
        \centering
        \hspace{-10pt}
        \includegraphics[width=1.0\linewidth]{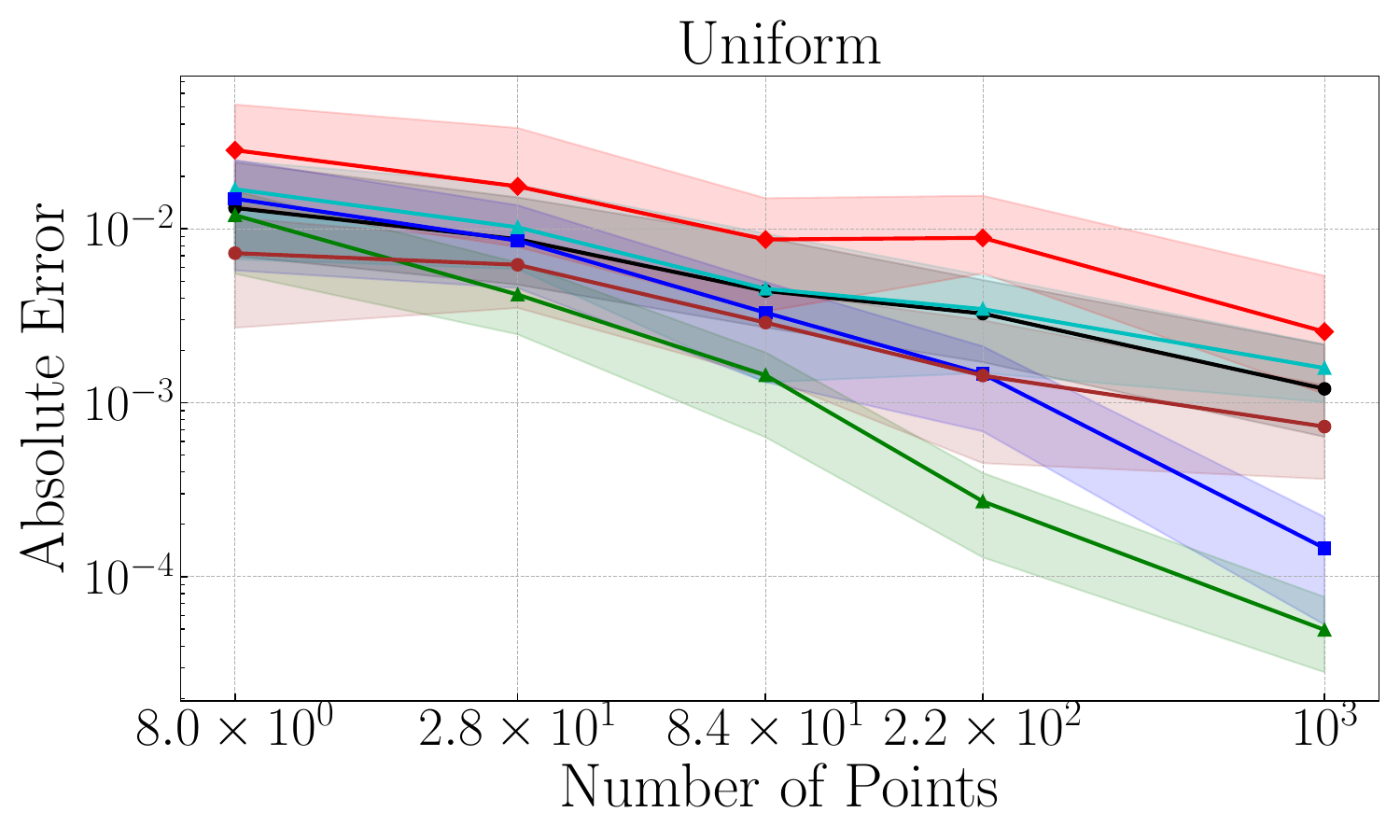}
        \label{fig:ising}
    \end{subfigure}%
    \hfill 
    \begin{subfigure}{0.33\textwidth}
        \centering
        \hspace{-10pt}
        \includegraphics[width=1.0\linewidth]{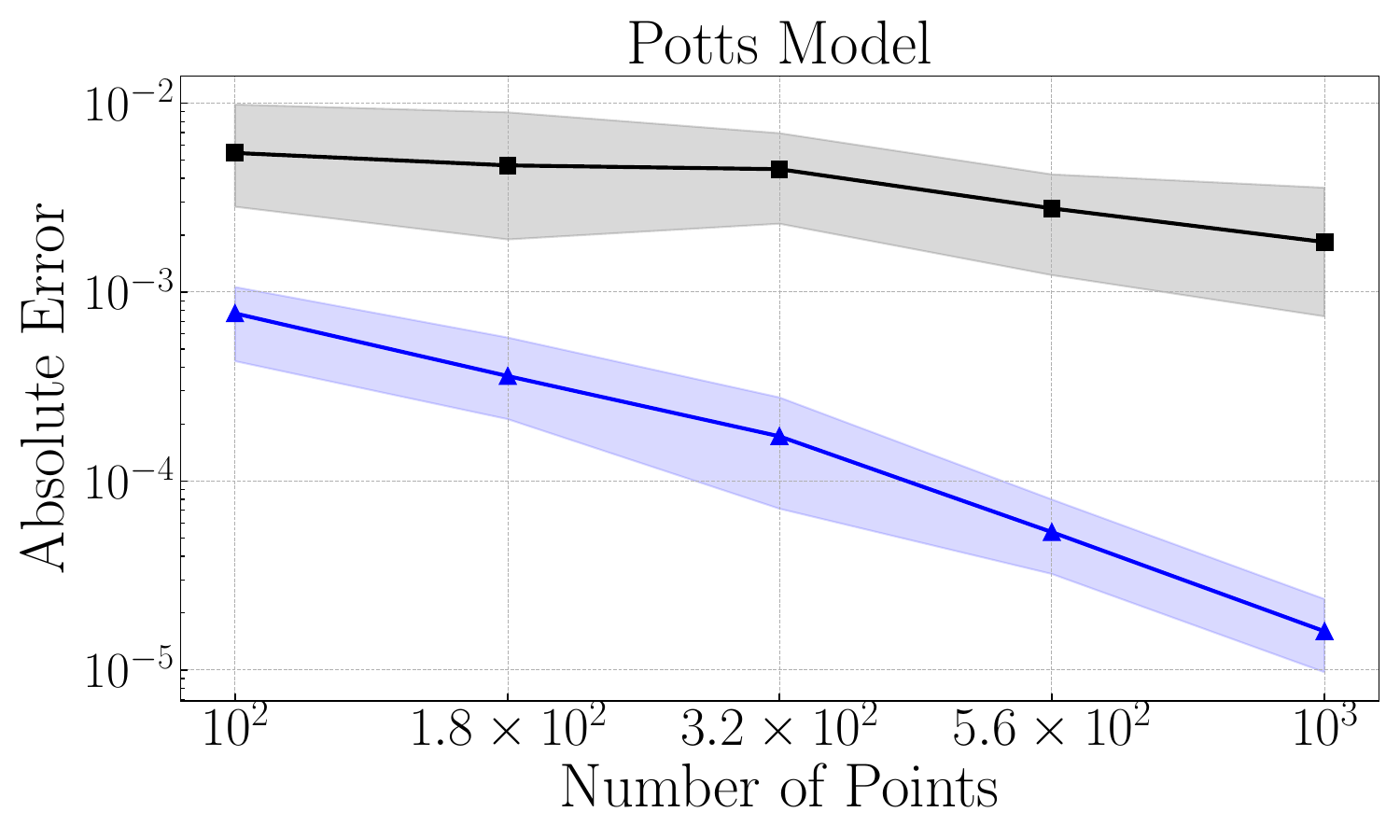}
        \label{fig:potts}
    \end{subfigure}
    \vspace{-3pt}
    \caption{Comparison of BayesSum / active BayesSum against baselines: Monte Carlo (MC), Russian roulette (RR), importance sampling (IS) and stratified sampling (SS). The reported results are absolute error on (a) Poisson distribution (\textbf{Left}), (b) uniform distribution over $\calX=\{0,1,2\}^L$ (\textbf{Middle}) and (c) un-normalized distribution characterized by a Potts model (\textbf{Right}). Results are averaged over 50 independent runs, while the shaded regions give the 25\%-75\% quantiles. }
    \label{fig:toy}
    \vspace{-15pt}
\end{figure*}

\subsection{Active BayesSum}\label{sec:active_dbq}
Rather than relying on a fixed set of samples, active learning can provide a principled framework for \emph{sequential} sample selection, where future query locations are chosen to maximize a user-specified \emph{utility} function. 
A natural choice for this utility is the expected information gain, which has been shown to substantially improve sample efficiency over naive randomized sampling in Bayesian quadrature~\citep{gunter2014sampling,osborne2012active,gessner2020active}. 
In our setting, the same utility can in principle be used to guide sample selection. 
The primary challenge, however, lies in solving the resulting optimization problem over the discrete domain. 

Let $\calD := \{x_i, f(x_i)\}_{i=1}^N$ denote the set of observed samples. Under a GP prior with mean $m \equiv 0$ and covariance kernel $k:\calX\times\calX\to\R$, the posterior GP conditioned on $\calD$ has some mean $\tilde{m}:\calX\to\R$ and covariance $\tilde{k}(x, x') := k(x,x') - k(x,x_{1:N})k(x_{1:N},x_{1:N})^{-1}k(x_{1:N}, x')$. Recall from \Cref{sec:dbq} that the BayesSum estimator induces a univariate Gaussian posterior distribution $\calN(\hat{I}, \hat{\sigma}^2)$. 
As a result, the mutual information between $I$ and a new observation $f(x^\ast)$ admits the following closed-form expression~\citep{gessner2020active}: 
\begin{align*}
    \alpha(x^\ast) = 
    -\log\bigg( 1 - \frac{( \mu_\Pb(x^\ast) - \E_{X\sim\Pb}[\tilde{k}(X, x^\ast)] )^2}{\hat{\sigma}^2 \tilde{k}(x^\ast, x^\ast)} \bigg) . 
\end{align*}

The main challenge is to identify the optimal next query $ \arg\max_{x\in\calX} \alpha(x^\ast)$ over the discrete domain $\calX$, where standard gradient-based optimization are not applicable and exhaustive evaluation of the utility across the entire $\calX$ is expensive. 
In our experiments, we select the next query by maximizing $\alpha(x^*)$ over a randomly chosen small subset $\mathcal{X}_{\text {sub }} \subset \mathcal{X}$ through exhaustive evaluation. 
Although straightforward, this approach may miss the globally optimal next query point. Another widely used alternative is to maximize the expected utility under an auxiliary distribution on $\mathcal{X}$, parameterized by continuous variables $\vartheta$~\citep{daulton2022bayesian,swersky2020amortized}.

\section{Experiments}\label{sec:experiment}
\vspace{-10pt}
We now illustrate BayesSum across a range of applications, including several synthetic examples, a Conway–Maxwell–Poisson model trained on count data, and a large Potts model trained on a set of sequences resembling biological proteins. 
The code to reproduce all our experiments are available at \url{https://github.com/seulkang0518/BayesSum}.

\subsection{Synthetic data}
First, we consider toy synthetic settings with the following pairs of function $f$ and distribution $\Pb$: 
\begin{enumerate}[label=(\alph*), itemsep=2.0pt, topsep=2pt, leftmargin=*]
    \item $\Pb$ is a Poisson distribution over $\calX=\N_+$ with rate parameter $\eta = 30$ and $f(x) = \exp(-(x - 15)^2 / 8)$.
    \item $\Pb$ is a uniform distribution over the state space $\calX=\{0, 1, 2\}^L$ and $f(x) = \exp(-\beta \sum_{i=1}^L h_i x_i - \beta \sum_{i\sim j} J_{ij} (x_i \ast x_j))$. Here, $x_i \ast x_j$ denotes an element-wise product between their one-hot encoded vectors and $i \sim j$ denotes adjacent positions on a $L\times L$ grid. The dimensionality of the original configuration space is $d = L$, and it becomes $d = 3L$ after one-hot encoding. 
    We set $h_i = 0.1 $, $J_{ij}= 0.1$ and $\beta = 1/2.269$ following \citet{moores2020scalable}. 
    \item $\Pb$ is a distribution specified by a Potts model over the state space $\calX=\{0, 1, 2\}^L$ where the probability mass function, known up to its normalization constant, is $p(x)\propto \exp(-\beta \sum_{i=1}^L h_i x_i - \beta \sum_{i\sim j} J_{ij} (x_i \ast x_j) )$ where we set $h_i = 0.1 $, $J_{ij}= 0.1$ and $\beta = 1/2.269$ same as above. The integrand function $f(x) = (1 + \exp(-\frac{1}{L} \sum_{i=1}^L x_i))^{-1}$. 
    \item $\Pb$ is the product of a uniform distribution over $\calX = [-1, 1]$ and a uniform distribution over $\calY = \left\{\, -1 + 0.125\,j \;:\; j = 0,1,\ldots,16 \,\right\}^2$. The dimensionality of this problem is $d = 3$. The function $f(x, y) = -20 \exp[- 0.2 |(1 + 0.1 y_1 + 0.1 y_2) x|] - \exp[\cos (2\pi x(1+0.05 y_1^2 + 0.05 y_2^2))] + 20 + e$. 
\end{enumerate}

Cases (a)–(d) cover all the various settings in which the proposed \textit{BayesSum} in \Cref{sec:dbq} are applicable: (a) intractable expectations over countably infinite domains; (b) expectations over finite but combinatorially large domains; (c) expectations over combinatorially large domains where the underlying distribution is known only up to a normalization constant; and (d) expectations over mixed domains involving both continuous and discrete components.
To enable extensive benchmarking against baseline methods, we set $L=15$ in cases (b) and (c), ensuring that the total number of possible states $3^{15}\approx 10^7$ remains computationally manageable such that the target expectation in \eqref{eq:I} can be computed exactly by explicitly summing over all possible states. 
For case (a), we truncate the infinite sum at $x=200$ 
and treat this value as the ground truth, as both the integrand $f(x)=\exp(-(x - 15)^2 / 8)$ and the Poisson distribution decays exponentially fast. 
Cases (d) admits a closed-form expression for its ground truth.

Regarding kernel choices for our BayesSum and active BayesSum, we employ a Brownian motion kernel in (a); an exponential Hamming distance kernel in (b); a Stein reproducing kernel built upon the exponential Hamming distance kernel in (c); and a product of the Gaussian kernel over the continuous domain and the exponential Hamming distance kernel over the discrete domain in (d). 
Closed-form expressions for the corresponding kernel mean embeddings are provided in \Cref{tab:discrete_kme}. 
For benchmarking, we compare against Monte Carlo (MC), Russian Roulette (RR), importance sampling (IS), and stratified sampling (SS). 
In cases (a), (b), and (d), MC samples are drawn i.i.d. with replacement, whereas BayesSum samples are drawn without replacement. This choice is intentional, as BayesSum suffers from singular Gram matrices with duplicate samples. 
(See \Cref{fig:iid_repetitive} for a comparison)
In case (c), samples for both BayesSum and MC are drawn from an unnormalized distribution using Metropolis–Hastings~\citep[Chapter 11]{gelman1995bayesian}. 
The empirical performances are reported by the mean absolute error with respect to the ground truth. 
Details of the baseline methods can be found in \Cref{sec:additional_experiments}.

\begin{figure*}[t]
    \centering
    \begin{subfigure}[t]{0.32\linewidth}
        \centering
        \includegraphics[width=\linewidth]{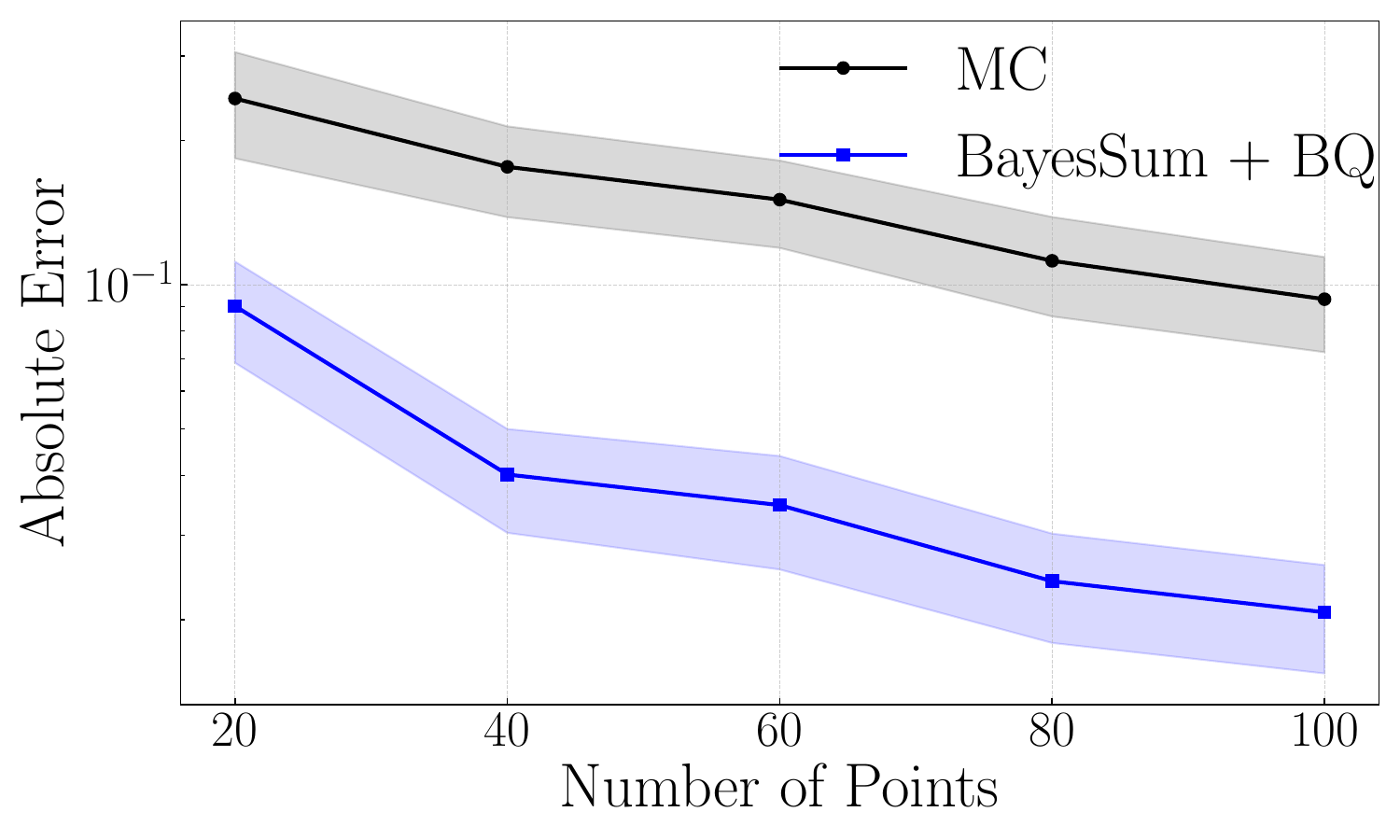}
    \end{subfigure}
    \hfill
    \begin{subfigure}[t]{0.32\linewidth}
        \centering
        \includegraphics[width=\linewidth]{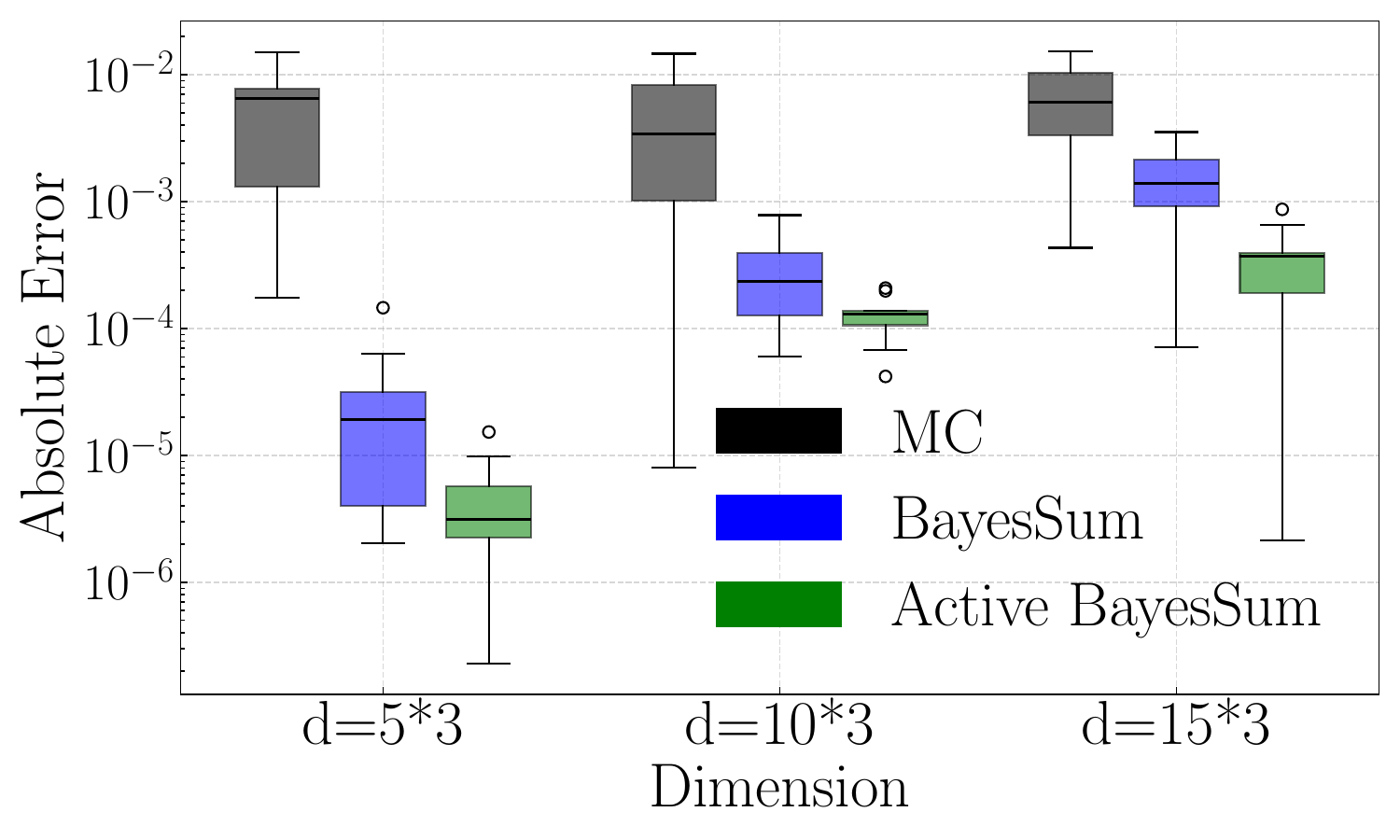}
    \end{subfigure}
    \hfill
    \begin{subfigure}[t]{0.32\linewidth}
        \centering
        \includegraphics[width=\linewidth]{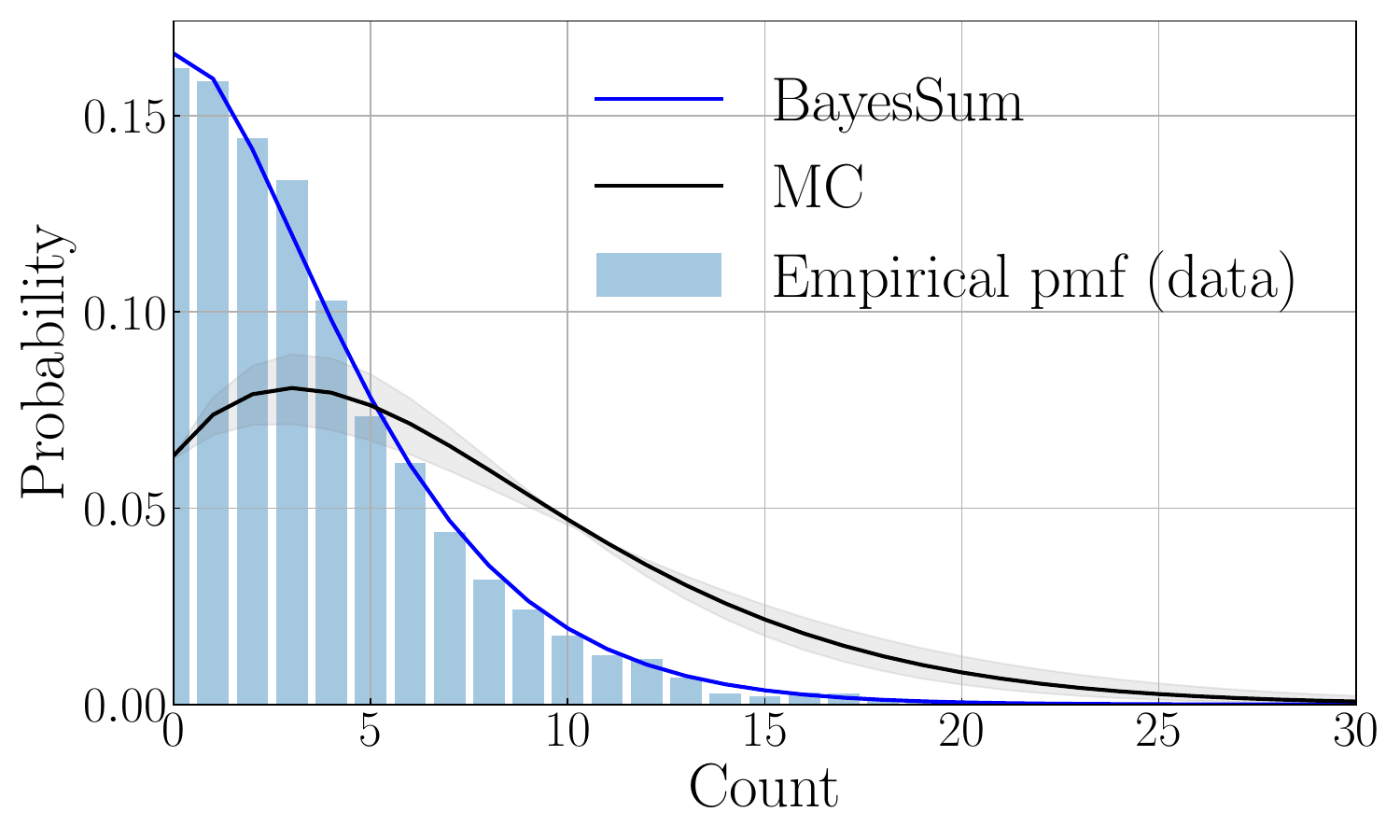}
    \end{subfigure}
    \vspace{-10pt}
    \caption{\textbf{Left:}  
    Comparison of BayesSum+BQ and MC over the mixed domain. \textbf{Middle:} Comparison of BayesSum, active BayesSum and MC with increasing dimension $d$ in case (b). \textbf{Right:} Comparison of BayesSum and MC in estimating the normalization constant for maximum likelihood inference of CMP model. BayesSum uses $N=10$ samples while MC uses $N=30$ samples. }
    \label{fig:calibration}
    \vspace{-10pt}
\end{figure*}

From \Cref{fig:toy} (\textbf{Left}) (\textbf{Middle}) (\textbf{Right}) and \Cref{fig:calibration} (\textbf{Left}), we observe that BayesSum achieves much smaller absolute errors than baseline methods for the same number of samples over all benchmarks from (a) to (d). 
This strong performance stems partly from the choice of kernel that encodes prior knowledge of the function $f$; in (a) the Brownian motion kernel captures smooth variations over the inputs, in (b)(c) the Hamming distance kernel captures smooth variations of $f$ across similar discrete sequences, and in (d), the product kernel jointly characterizes smoothness over the continuous domain and similarity over the discrete domain. 
We also observe that active BayesSum can further improve the sample efficiency of BayesSum, by selecting the next query point that drives the largest reduction in posterior uncertainty. 
We also compare in \Cref{fig:iid_repetitive} the performance of BayesSum with repetitive versus non-repetitive samples. We observe that repetitive sampling leads to notably larger errors, which verifies \Cref{thm:main} and \Cref{rem:iid_repeat}. 

To further illustrate the sample efficiency, we estimate the empirical convergence rate of all methods by performing a linear regression in log–log space between the number of samples and the mean absolute error. 
Note that for Poisson distribution, the rate is obtained with far more number of samples than plotted in \Cref{fig:toy} (\textbf{Left}). 
We observe that MC, RR, SS, and IS all exhibit similar convergence rates around $0.5$. This is consistent with the central limit theorem for i.i.d. samples, as well as for samples drawn from a well-mixed Markov chain~\citep{mcbook}. 
In contrast, \Cref{tab:empirical_rate} shows that both our BayesSum and Active BayesSum achieve substantially faster rates of convergence, confirming their superiority in sample efficiency.

In \Cref{fig:calibration} (\textbf{Middle}), we also study the performance of BayesSum, active BayesSum in case (b) with increasing dimension $d$. We observe that the performance margin over the baselines gets smaller as $d$ increases. This is expected, since Bayesian quadrature or Gaussian process are known to suffer in high dimensions~\citep{chen2024conditional,briol2019probabilistic}, and a similar phenomenon is anticipated in discrete domains. 
We also study the calibration of the BayesSum posterior variance in \Cref{fig:toy_additional}. We observe that the posterior variances for the uniform distribution over $\{0,1,2\}^d$ in case (b) and the mixed distribution in case (d) are well-calibrated, whereas the posterior variance for the Poisson distribution in (a) and for the Potts model in (d) are underconfident. 
This behavior arises because the Brownian kernel and the Stein kernel are unbounded, which leads to inflated posterior variances even with carefully tuned kernel amplitudes. 

\begin{table}[t]
\centering
\caption{Comparison of empirical convergence rate of different methods. 
For each method, we report the estimated $\alpha > 0$ such that the integration error scales as $O(n^{-\alpha})$. 
Larger values indicate faster convergence. }
\label{tab:empirical_rate}
\vspace{-10pt}
\begin{tabular}{lccc}
\toprule
\textbf{Method} & \textbf{Poisson} & \textbf{Uniform} & \textbf{Potts} \\
\midrule
MC & $0.455$ & $0.530$ & $0.540$ \\
\textcolor{red}{RR} & $0.335$ & $0.515$ & --- \\
\textcolor{cyan}{SS} & $0.454$ & $0.552$ & --- \\
\textcolor{brown}{IS} & $0.515$ & $0.545$ & --- \\
\textcolor{blue}{BayesSum} & $7.322$ & $1.139$ & $\boldsymbol{1.800}$ \\
\textcolor{ForestGreen}{Active BayesSum} & $\boldsymbol{7.706}$ & $\boldsymbol{1.251}$ & --- \\
\bottomrule
\end{tabular}
\vspace{-10pt}
\end{table}

\subsection{Parameter estimation
for unnormalised models}

\textbf{Conway–Maxwell–Poisson model for count data:} The Conway–Maxwell–Poisson (CMP) distribution is a generalization of the Poisson distribution over $\calX=\N$~\citep{conway1961queueing}, which has been widely used for count data in a wide range of fields including transport, finance and retail~\citep{inouye2017review,benson2021bayesian,sellers2010flexible}. 
The CMP distribution has two parameters $(\theta_1, \theta_2) \in (0, \infty) \times [0,\infty)$ and its probability mass function $p(x; \theta_1,\theta_2)$ is known only up to a normalization constant $Z(\theta_1,\theta_2)$: 
\begin{align*}
    p(x; \theta_1,\theta_2)= \frac{\theta_1^x (x!)^{-\theta_2}}{ Z(\theta_1,\theta_2)}, \quad Z(\theta_1,\theta_2) =\sum_{j=0}^{\infty} \theta_1^j (j!)^{-\theta_2} .
\end{align*}
Our goal is to model the sales dataset of \citet{shmueli2005useful} with a CMP distribution with parameters $(\theta_1, \theta_2)$. To estimate these parameters via gradient-based maximum likelihood, it is necessary to compute the normalization constant $Z(\theta_1,\theta_2)$. 
To enable the use of BayesSum, we first rewrite the normalization constant $Z(\theta_1, \theta_2)$ as an expectation under a Poisson distribution with rate parameter $\eta_0$: $Z(\theta_1, \theta_2) = e^{\eta_0}  \E_{X \sim\text{Poisson}(\eta_0)}[\frac{\theta_1}{\eta_0}(X!)^{1-\theta_2}]$, where $\eta_0$ is set to be the empirical mean of the data.
The key reason we write the normalization constant as an expectation with respect to a fixed distribution is computational: we only need to draw $N = 10$ samples without replacement and to compute the BayesSum weights $w_{1:N} = \mu_\Pb(x_{1: N})^{\top} \mathbf{K}^{-1}$ once at initialization. 
We use the Brownian motion kernel that admits a closed form kernel mean embedding $\mu_\Pb$. 
Thereafter, estimating the normalization constant at each iteration amounts to a weighted sum $\hat{Z}_\bq(\theta_1, \theta_2)=\sum_{i=1}^N w_i \frac{\theta_1}{\eta_0} (x_i)!^{1-\theta_2}$, making BayesSum similar computational cost as MC. 

With this approximation, the parameters $(\theta_1, \theta_2)$ are then updated via gradient descent on the estimated negative log-likelihood using a learning rate of $1 \times 10^{-3}$ for 800 iterations where convergence is empirically observed. 
For comparison, we perform the same training procedure where $Z(\theta_1, \theta_2)$ is instead approximated by a Monte Carlo estimator $\hat{Z}_{\text{MC}}(\theta_1, \theta_2)$ with $N=30$ samples at each iteration. 
As shown in \Cref{fig:calibration} \textbf{(Right)}, maximum-likelihood training with  $\hat{Z}_{\text{BQ}}$ yields parameters $(\theta_1,\theta_2)$ that better capture the true distribution of the sales dataset than those obtained with $\hat{Z}_{\text{MC}}$. 
This improvement arises because our BayesSum provides a more accurate estimate of the normalization constant even with fewer samples, and consequently, of the gradient of the log-likelihood at each iteration.
This can be further reflected in \Cref{fig:trajectory} where we plot the optimization trajectories of $(\theta_1, \theta_2)$ for BayesSum and Monte Carlo. 
The computational time is $4.2$s for BayesSum and $4.0$s for Monte Carlo. 

\textbf{Potts model for sequence data: }
The Potts model is a popular model for protein sequences due to its explicit representation of pairwise couplings where mutations at one site are often compensated by changes at another.
Formally, the Potts model defines the likelihood of a sequence $x = [x_1, \ldots, x_L] \in \{1, 2, \ldots, S\}^L$ as
\begin{align}\label{eq:potts}
    &p(x; \boldsymbol{\theta}) = \frac{\exp \left( \sum_{i=1}^L \beta h_i x_i + \sum_{i \sim j} \beta J_{i j}(x_i \ast x_j) \right)}{Z(\boldsymbol{\theta})}\nonumber \\
    &Z(\boldsymbol{\theta}) = \sum_{x} \exp \left( \sum_{i=1}^L \beta h_i x_i + \sum_{i\sim j} J_{i j}(x_i \ast x_j) \right) .
\end{align}
Here, the parameter vector $\boldsymbol{\theta}$ consists of two components: $\mathbf{h} = [h_1,\ldots,h_L]^\top \in \mathbb{R}^L$, which denotes the marginal fields, and $\mathbf{J} = [J_{i,j}] \in \mathbb{R}^{L \times L}$, which encodes the pairwise interaction strengths between $x_i$ and $x_j$. 
The symbol $\ast$ denotes an element-wise product between their one-hot encoded vectors. 
$\beta = 1/2.269$ represents the inverse temperature. 
The normalization constant $Z(\boldsymbol{\theta})$ is the sum over $L^S$ terms, which quickly becomes intractable as $S$ and $L$ grow. 

\begin{figure}[t]
    \centering
    \includegraphics[width=0.75\linewidth]{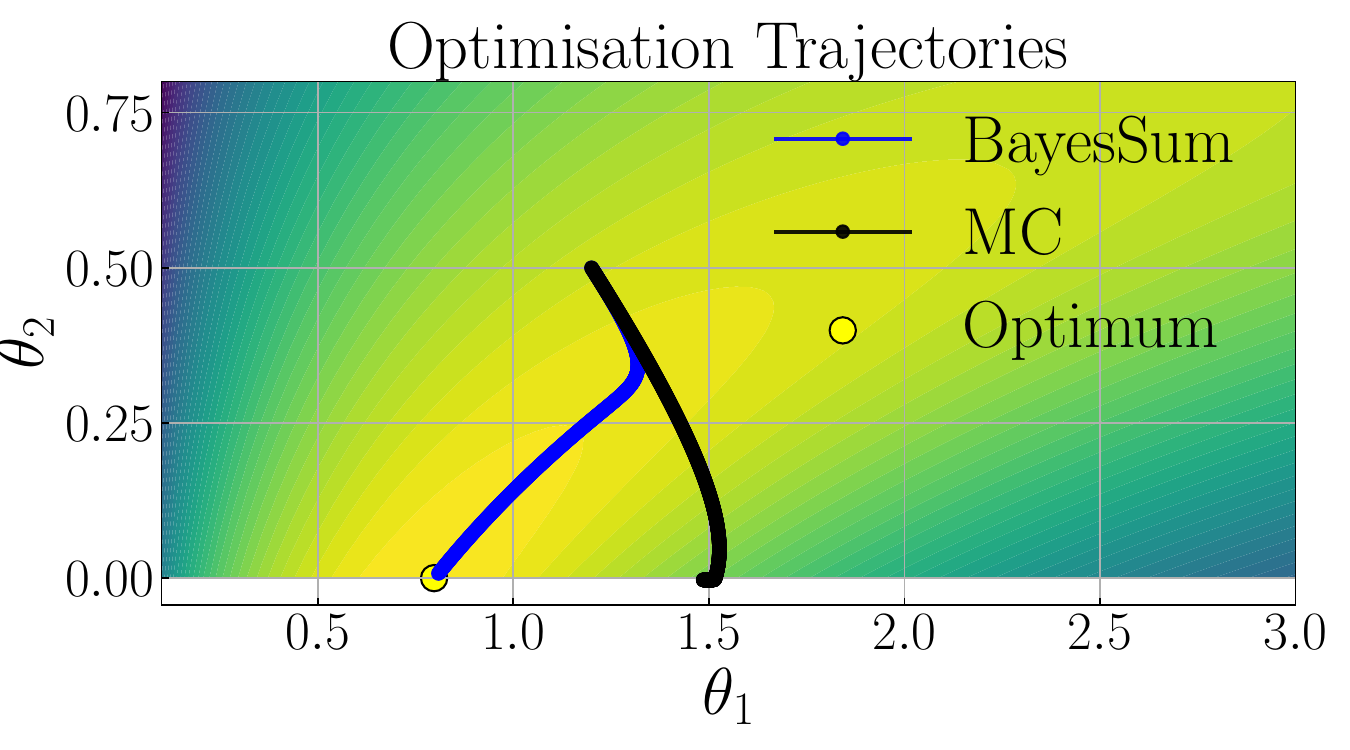}
    \vspace{-10pt}
    \caption{The optimization trajectory of training CMP model with normalization constant estimated via BayesSum and MC.}
    \label{fig:trajectory}
    \vspace{-20pt}
\end{figure}

We train the Potts model via maximum likelihood using $2000$ protein sequences of length $L = 15$ generated as follows: each symbol is independently drawn from the alphabet $\{1, 2, 3\}$ with probabilities $\{0.4, 0.4, 0.2\}$. After one-hot encoding, the total dimension of the problem becomes $d = 15 \times 3 = 45$. 
Same as the CMP distribution above, 
we perform gradient descent on the log-likelihood and estimate the intractable normalization constant $Z$ at each iteration. 
To this end, conditioned on a sample $x^\prime$ from the last iteration, we draw $\lceil N / M \rceil$ samples from a product distribution $\prod_{j=1}^L \bar{p}_j(x_j \mid x')$ with 
\begin{align*}
     \bar{p}_j(x_j\mid x'; \mathbf{h}, \mathbf{J}) := \text{softmax}(\beta h_j + \beta \sum_{i \sim j} J_{ij}(x_i' \ast x_j) ) .  
\end{align*}
This product distribution is motivated by the pseudo-likelihood approach, which replaces the joint distribution with a product of conditional distributions~\citep{ekeberg2013improved}. 
We repeat this procedure for $M$ independent choices of $x^{\prime}$, yielding a total of $N$ samples.
Then, we approximate $Z(\boldsymbol{\theta})$ using our BayesSum estimator $\hat{Z}_{\bq}$ from these $N$ samples. 
With this approximation, then we are able to optimize the parameters $\mathbf{h}, \mathbf{J}$ by performing gradient descent on the estimated negative log-likelihood, using a learning rate of $10^{-3}$ for $2000$ iterations until convergence is empirically observed. 
For comparison, we repeat the same procedure as above yet with $Z(\boldsymbol{\theta})$ estimated by a Monte Carlo estimator $\hat{Z}_{\text{MC}}$ from the same $N$ samples as BayesSum at each iteration.

In \Cref{fig:potts_protein}, we compare the two Potts models trained using the above two procedures measured in terms of the squared discrete kernel Stein discrepancy (KSD$^2$)~\citep{yang2018goodness} with respect to the true data generating distribution. 
\Cref{fig:potts_protein} shows that the final Potts model trained with $\hat{Z}_{\bq}$ attains much lower KSD$^2$ than the model trained with $\hat{Z}_{\text{MC}}$, both under an equal computational budget in wall-clock time (\textbf{top}) and under an equal number of samples (\textbf{bottom}). 
This result is consistent with case~(b) of the synthetic experiments, demonstrating that the benefit of employing a more accurate estimator of the normalization constant would translate into a better trained model.

\begin{figure}[t]
    \centering
    \begin{subfigure}[t]{0.8\linewidth}
        \centering
        \includegraphics[width=\linewidth]{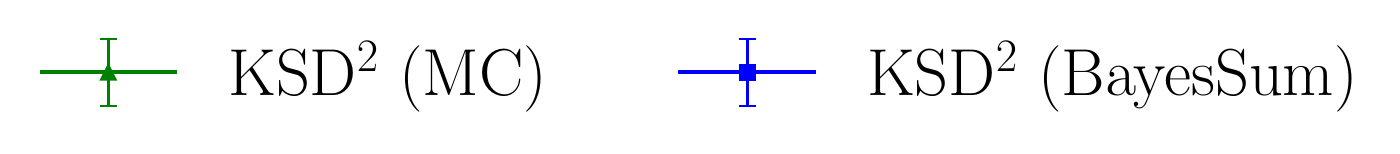}
    \end{subfigure}
    \begin{subfigure}[t]{0.9\linewidth}
        \centering
        \includegraphics[width=\linewidth]{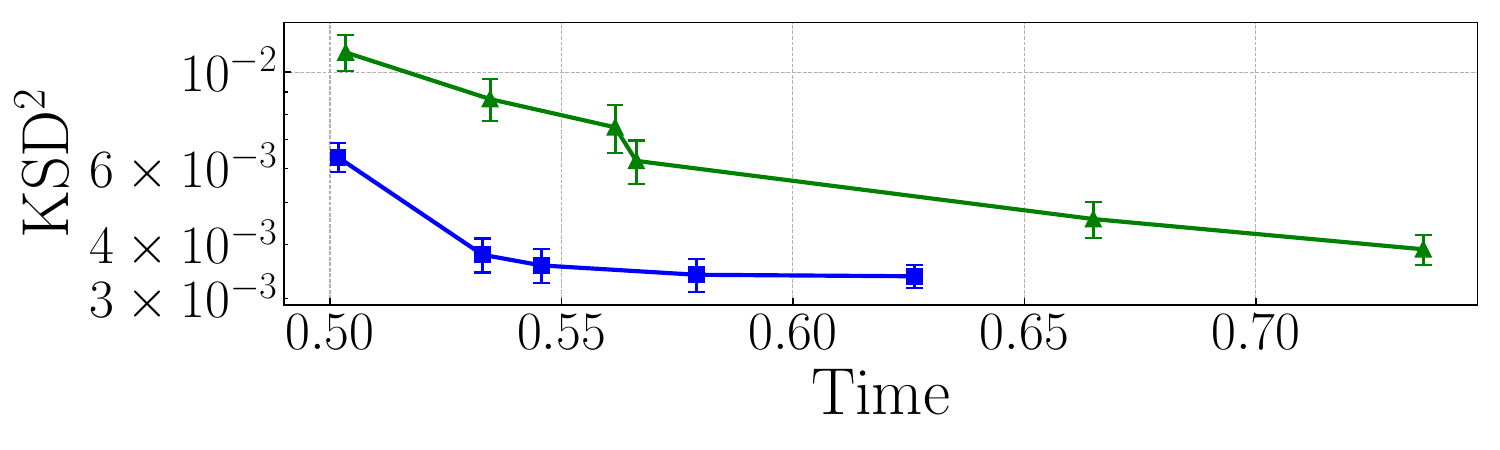}
    \end{subfigure}
    \vspace{6pt}
    \begin{subfigure}[t]{0.9\linewidth}
        \centering
        \includegraphics[width=\linewidth]{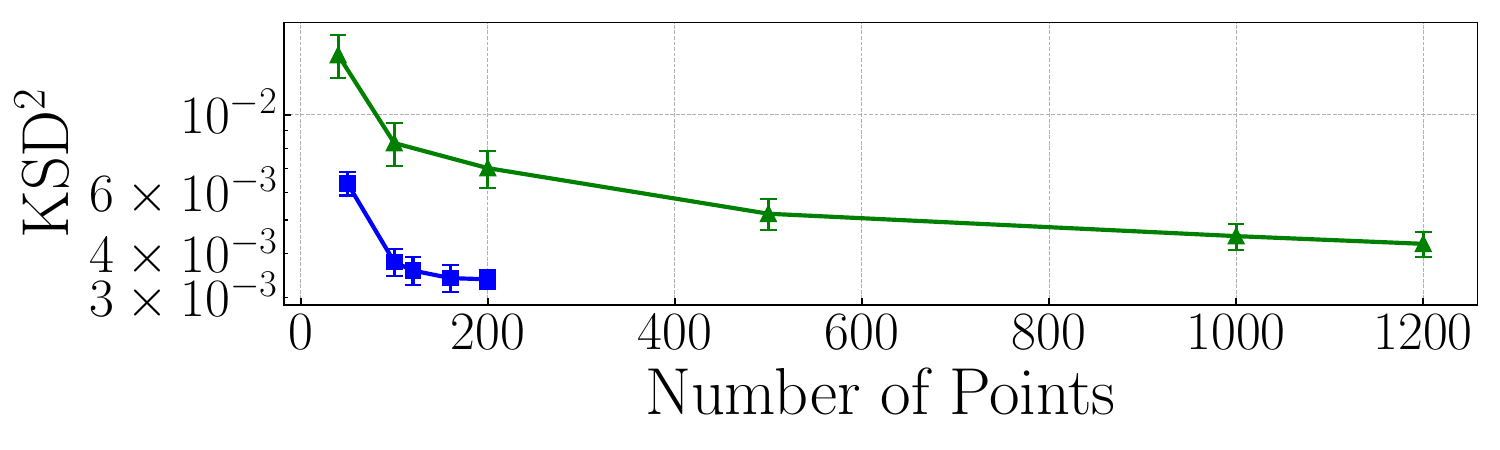}
    \end{subfigure}
    \vspace{-10pt}
    \caption{Comparison of a large Potts model trained via maximum log likelihood, with the normalization constant estimated by BayesSum and Monte Carlo, under the same computational time (\textbf{top}) and sample size $N$ (\textbf{bottom}). Error bars represent the standard error repeated over $30$ random seeds. }
    \label{fig:potts_protein}
    \vspace{-20pt}
\end{figure}

\section{Conclusions}\label{sec:conclusion}
In this paper, we introduced BayesSum, a novel algorithm for estimating intractable expectations over discrete domains. BayesSum offers improved sample efficiency and provides finite-sample uncertainty quantification for its estimates, making it particularly attractive in scenarios where obtaining samples or function evaluations are expensive. 

Following our work, there remains a number of interesting open problems. First, it would be valuable to establish convergence rate of BayesSum where the key theoretical challenge lies in identifying families of kernels whose RKHSs are `equivalent' to Sobolev spaces defined over discrete domains~\citep{stevenson1991discrete}, while still admitting closed-form kernel mean embeddings.
Furthermore, a natural extension would be to apply BayesSum to  multiple-related expectation, conditional expectations, or nested expectations that frequently arise in hierarchical Bayesian inference and decision-making under uncertainty~\citep{chen2024conditional,chennested,xi2018bayesian}. Finally, alternative to Gaussian processes priors have shown promise for integration in the continuous setting \citep{Oates2017heart,Zhu2020,Ott2023} and it could be interesting to study discrete equivalents.

\subsubsection*{Acknowledgments}
The authors acknowledge support from the UK's Engineering and Physical
Sciences Research Council (ESPRC) through grants
[EP/S021566/1] (for ZC) and [EP/Y022300/1]
(for FXB).
TK was supported by the Research Council of Finland grants 359183 and 368086.
TK acknowledges the research environment provided by ELLIS Institute Finland.

\bibliographystyle{plainnat}  
\bibliography{main}  

@article{conway1961queueing,
  title={A queueing model with state dependent service rate},
  author={R. W. Conway},
  journal={Journal of Industrial Engineering},
  volume={12},
  pages={132},
  year={1961}
}

@article{karvonen2018bayes,
  title={A {Bayes-Sard} cubature method},
  author={Karvonen, T. and Oates, C. J. and Sarkka, S.},
  journal={Advances in Neural Information Processing Systems},
  volume={31},
  year={2018}
}

@book{davis2007methods,
  title={Methods of numerical integration},
  author={Davis, P. J. and Rabinowitz, P.},
  year={2007},
  publisher={Courier Corporation}
}

@article{stevenson1991discrete,
  title={Discrete {Sobolev} spaces and regularity of elliptic difference schemes},
  author={Stevenson, R.},
  journal={ESAIM: Mathematical Modelling and Numerical Analysis},
  volume={25},
  number={5},
  pages={607--640},
  year={1991},
  publisher={EDP Sciences}
}

@article{fukumizu2011kernel,
  title={{Kernel Bayes'} rule},
  author={Fukumizu, K. and Song, L. and Gretton, A.},
  journal={Advances in Neural Information Processing Systems},
  volume={24},
  year={2011}
}

@article{nishiyama2020model, title={Model-based kernel sum rule: kernel {Bayesian} inference with probabilistic models}, author={Nishiyama, Y. and Kanagawa, M. and Gretton, A. and Fukumizu, K.}, journal={Machine Learning}, volume={109}, number={5}, pages={939--972}, year={2020}, publisher={Springer} }

@book{gelman1995bayesian,
  title={Bayesian data analysis},
  author={Gelman, A. and Carlin, J. B. and Stern, H. S. and Rubin, D. B.},
  year={1995},
  publisher={Chapman and Hall/CRC}
}

@article{karvonen2019symmetry,
  title={Symmetry exploits for {Bayesian} cubature methods},
  author={Karvonen, T. and Sarkka, S. and Oates, C. J.},
  journal={Statistics and Computing},
  volume={29},
  number={6},
  pages={1231--1248},
  year={2019},
  publisher={Springer}
}

@article{epperly2023kernel,
  title={Kernel quadrature with randomly pivoted {Cholesky}},
  author={Epperly, E. and Moreno, E.},
  journal={Advances in Neural Information Processing Systems},
  volume={36},
  pages={65850--65868},
  year={2023}
}

@inproceedings{hayakawa2023sampling,
  title={Sampling-based {Nystr{\"o}m} approximation and kernel quadrature},
  author={Hayakawa, S. and Oberhauser, H. and Lyons, T.},
  booktitle={International Conference on Machine Learning},
  pages={12678--12699},
  year={2023},
  organization={PMLR}
}

@article{hayakawa2022positively,
  title={Positively weighted kernel quadrature via subsampling},
  author={Hayakawa, S. and Oberhauser, H. and Lyons, T.},
  journal={Advances in Neural Information Processing Systems},
  volume={35},
  pages={6886--6900},
  year={2022}
}

@article{kuo2025constructing,
  title={Constructing embedded lattice-based algorithms for multivariate function approximation with a composite number of points},
  author={Kuo, F. Y. and Mo, W. and Nuyens, D.},
  journal={Constructive Approximation},
  volume={61},
  number={1},
  pages={81--113},
  year={2025},
  publisher={Springer}
}

@article{karvonen2018fully,
  title={Fully symmetric kernel quadrature},
  author={Karvonen, T. and Sarkka, S.},
  journal={SIAM Journal on Scientific Computing},
  volume={40},
  number={2},
  pages={A697--A720},
  year={2018},
  publisher={SIAM}
}

@book{steinwart2008support,
  title={Support vector machines},
  author={Steinwart, I. and Christmann, A.},
  year={2008},
  publisher={Springer Science \& Business Media}
}

@article{fortuin2021sparse,
  title={{Sparse Gaussian} processes on discrete domains},
  author={Fortuin, V. and Dresdner, G. and Strathmann, H. and Ratsch, G.},
  journal={IEEE Access},
  volume={9},
  pages={76750--76758},
  year={2021},
  publisher={IEEE}
}

@inproceedings{ru2020bayesian, title={Bayesian optimisation over multiple continuous and categorical inputs}, author={Ru, B. and Alvi, A. and Nguyen, V. and Osborne, M. A. and Roberts, S.}, booktitle={International Conference on Machine Learning}, pages={8276--8285}, year={2020}, organization={PMLR} }

@article{kanagawa2018gaussian,
  title={Gaussian processes and kernel methods: A review on connections and equivalences},
  author={Kanagawa, M. and Hennig, P. and Sejdinovic, D. and Sriperumbudur, B. K.},
  journal={arXiv preprint arXiv:1807.02582},
  year={2018}
}

@article{balandat2020botorch,
  title={{BoTorch}: A framework for efficient {Monte Carlo} {Bayesian Optimization}},
  author={Balandat, M. and Karrer, B. and Jiang, D. and Daulton, S. and Letham, B. and Wilson, A. G. and Bakshy, E.},
  journal={Advances in Neural Information Processing Systems},
  volume={33},
  pages={21524--21538},
  year={2020}
}

@book{gartner2008kernels,
  title={Kernels for structured data},
  author={Gartner, T.},
  volume={72},
  year={2008},
  publisher={World Scientific}
}

@article{shmueli2005useful,
  title={A useful distribution for fitting discrete data: revival of the {Conway-Maxwell-Poisson} distribution},
  author={Shmueli, G. and Minka, T. P. and Kadane, J. B. and Borle, S. and Boatwright, P.},
  journal={Journal of the Royal Statistical Society Series C: Applied Statistics},
  volume={54},
  number={1},
  pages={127--142},
  year={2005},
  publisher={Oxford University Press}
}

@article{inouye2017review,
  title={A review of multivariate distributions for count data derived from the {Poisson} distribution},
  author={Inouye, D. I. and Yang, E. and Allen, G. I. and Ravikumar, P.},
  journal={Wiley Interdisciplinary Reviews: Computational Statistics},
  volume={9},
  number={3},
  pages={e1398},
  year={2017},
  publisher={Wiley Online Library}
}

@article{lyne2015russian,
  title={On {Russian} roulette Estimates for {Bayesian} Inference with Doubly-Intractable Likelihoods},
  author={Lyne, A.-M. and Girolami, M. and Atchad{\'e}, Y. and Strathmann, H. and Simpson, D.},
  journal={Statistical Science},
  volume={30},
  number={4},
  pages={443--467},
  year={2015}
}

@article{briol2025dictionary,
author = {Briol, F.-X. and Gessner, A. and Karvonen, T. and Mahsereci, M.},
journal = {Proceedings of the First International Conference on Probabilistic Numerics, PMLR},
pages = {84--94},
title = {{A dictionary of closed-form kernel mean embeddings}},
volume = {271},
year = {2025}
}

@inproceedings{hendricks2006mcnp,
  title={{MCNP} variance reduction overview},
  author={Hendricks, J. S. and Booth, T. E.},
  booktitle={Monte-Carlo Methods and Applications in Neutronics, Photonics and Statistical Physics},
  pages={83--92},
  year={2006},
  organization={Springer}
}

@article{cockayne2019bayesian,
  title={Bayesian probabilistic numerical methods},
  author={Cockayne, J. and Oates, C. J. and Sullivan, T. J. and Girolami, M.},
  journal={SIAM Review},
  volume={61},
  number={4},
  pages={756--789},
  year={2019},
  publisher={SIAM}
}

@book{wendland2004scattered,
  title={Scattered data approximation},
  author={Wendland, H.},
  volume={17},
  year={2004},
  publisher={Cambridge University Press}
}

@article{owen2003quasi,
  title={{Quasi-Monte Carlo} sampling},
  author={Owen, A. B.},
  journal={Monte Carlo Ray Tracing: Siggraph},
  volume={1},
  pages={69--88},
  year={2003},
  publisher={Citeseer}
}

@article{caflisch1998monte,
  title={{Monte Carlo} and {quasi-Monte Carlo} methods},
  author={Caflisch, R. E.},
  journal={Acta Numerica},
  volume={7},
  pages={1--49},
  year={1998},
  publisher={Cambridge University Press}
}

@book{mcbook,
   author = {Owen, A. B.},
   year = 2013,
   title = {{Monte Carlo} theory, methods and examples},
   howpublished = {\url{https://artowen.su.domains/mc/}}
}

@book{williams2006gaussian,
  title={Gaussian processes for machine learning},
  author={Williams, C. K. I. and Rasmussen, C. E.},
  volume={2},
  number={3},
  year={2006},
  publisher={MIT Press, Cambridge, MA}
}

@book{rubinstein2016simulation,
  title={Simulation and the {Monte Carlo} method},
  author={Rubinstein, R. Y. and Kroese, D. P.},
  year={2016},
  publisher={John Wiley \& Sons}
}

@article{piancastelli2023multivariate,
  title={Multivariate {Conway-Maxwell-Poisson} distribution: {Sarmanov} method and doubly intractable {Bayesian} inference},
  author={Piancastelli, L. S. C. and Friel, N. and Barreto-Souza, W. and Ombao, H.},
  journal={Journal of Computational and Graphical Statistics},
  volume={32},
  number={2},
  pages={483--500},
  year={2023},
  publisher={Taylor \& Francis}
}

@book{baxter2016exactly,
  title={Exactly solved models in statistical mechanics},
  author={Baxter, R. J.},
  year={2016},
  publisher={Elsevier}
}

@article{Diaconis1988,
  title={{B}ayesian numerical analysis},
  author={Diaconis, P.},
  journal={Statistical Decision Theory and Related Topics IV},
  volume={1},
  pages={163--175},
  year={1988}
}

@article{o1991bayes,
  title={{Bayes-Hermite} quadrature},
  author={O'Hagan, A.},
  journal={Journal of Statistical Planning and Inference},
  volume={29},
  number={3},
  pages={245--260},
  year={1991},
  publisher={Elsevier}
}

@article{ekeberg2013improved,
  title={Improved contact prediction in proteins: using pseudolikelihoods to infer Potts models},
  author={Ekeberg, M. and L{\"o}vkvist, C. and Lan, Y. and Weigt, M. and Aurell, E.},
  journal={Physical Review E—Statistical, Nonlinear, and Soft Matter Physics},
  volume={87},
  number={1},
  pages={012707},
  year={2013},
  publisher={APS}
}

@article{rasmussen2003bayesian,
  title={Bayesian {Monte Carlo}},
  author={Rasmussen, C. E. and Ghahramani, Z.},
  journal={Advances in Neural Information Processing Systems},
  pages={505--512},
  year={2003}, 
  publisher={MIT; 1998}
}

@article{Karvonen2020,
author = {Karvonen, T. and Oates, C. J. and Girolami, M.},
journal = {Mathematics of Computation},
number = {331},
pages = {2209--2233},
title = {{Integration in reproducing kernel Hilbert spaces of Gaussian kernels}},
volume = {90},
year = {2020}
}

@article{oates2017control,
  title={Control functionals for Monte Carlo integration},
  author={Oates, C. J. and Girolami, M. and Chopin, N.},
  journal={Journal of the Royal Statistical Society Series B: Statistical Methodology},
  volume={79},
  number={3},
  pages={695--718},
  year={2017},
  publisher={Oxford University Press}
}

@inproceedings{Bharti2023,
author = {Bharti, A. and Naslidnyk, M. and Key, O. and Kaski, S. and Briol, F-X.},
booktitle = {International Conference on Machine Learning},
pages = {2289--2312},
title = {{Optimally-weighted estimators of the maximum mean discrepancy for likelihood-free inference}},
year = {2023}
}

@inproceedings{Oates2017heart,
author = {Oates, C. J. and Niederer, S. and Lee, A. and Briol, F-X. and Girolami, M.},
booktitle = {Neural Information Processing Systems},
pages = {110--118},
title = {{Probabilistic models for integration error in the assessment of functional cardiac models}},
year = {2017}
}

@inproceedings{Zhu2020,
author = {Zhu, H and Liu, X. and Kang, R. and Shen, Z. and Flaxman, S. and Briol, F-X.},
booktitle = {Neural Information Processing Systems},
pages = {5837--5849},
title = {{Bayesian probabilistic numerical integration with tree-based models}},
year = {2020}
}

@inproceedings{Ott2023,
author = {Ott, K. and Tiemann, M. and Hennig, P. and Briol, F-X.},
booktitle = {Proceedings of the Thirty-Ninth Conference on Uncertainty in Artificial Intelligence, PMLR 216},
pages = {1606--1617},
title = {{Bayesian numerical integration with neural networks}},
year = {2023}
}

@inproceedings{huszar2012optimally,
  title={Optimally-weighted herding is Bayesian quadrature},
  author={Husz{\'a}r, F. and Duvenaud, D.},
  booktitle={Proceedings of the Twenty-Eighth Conference on Uncertainty in Artificial Intelligence},
  pages={377--386},
  year={2012}
}

@inproceedings{sun2023vector,
  title={Vector-valued control variates},
  author={Sun, Z. and Barp, A. and Briol, F.-X.},
  booktitle={International Conference on Machine Learning},
  pages={32819--32846},
  year={2023},
  organization={PMLR}
}

@article{sellers2010flexible,
  title={A flexible regression model for count data},
  author={Sellers, Kimberly F and Shmueli, Galit},
  journal={The Annals of Applied Statistics},
  pages={943--961},
  year={2010},
  publisher={JSTOR}
}

@inproceedings{Kanagawa2016,
author = {Kanagawa, M. and Sriperumbudur, B. and Fukumizu, K.},
booktitle = {Advances in Neural Information Processing Systems},
pages = {3288--3296},
title = {{Convergence guarantees for kernel-based quadrature rules in misspecified settings}},
year = {2016}
}

@article{Wynne2020,
author = {Wynne, G. and Briol, F.-X. and Girolami, M.},
journal = {Journal of Machine Learning Research},
number = {123},
pages = {1--40},
title = {{Convergence guarantees for Gaussian process means with misspecified likelihoods and smoothness}},
volume = {22},
year = {2021}
}

@inproceedings{Kanagawa2019,
archivePrefix = {arXiv},
arxivId = {1905.10271},
author = {Kanagawa, M. and Hennig, P.},
booktitle = {Advances in Neural Information Processing Systems},
eprint = {1905.10271},
pages = {6237--6248},
title = {{Convergence guarantees for adaptive Bayesian quadrature methods}},
year = {2019}
}

@article{benson2021bayesian,
  title={Bayesian inference, model selection and likelihood estimation using fast rejection sampling: the {C}onway-{M}axwell-{P}oisson distribution},
  author={Benson, A. and Friel, N.},
  journal={Bayesian Analysis},
  volume={16},
  number={3},
  pages={905--931},
  year={2021},
  publisher={International Society for Bayesian Analysis}
}

@inproceedings{Chen2010,
author = {Chen, Y. and Welling, M. and Smola, A.},
booktitle = {Proceedings of the Conference on Uncertainty in Artificial Intelligence},
pages = {109--116},
title = {{Super-samples from kernel herding}},
year = {2010}
}

@inproceedings{Chen2018,
author = {Chen, W. Y. and Mackey, L. and Gorham, J. and Briol, F-X. and Oates, C. J.},
booktitle = {Proceedings of the International Conference on Machine Learning, PMLR 80},
pages = {843--852},
title = {{Stein points}},
year = {2018}
}

@article{shi2022gradient,
  title={Gradient estimation with discrete stein operators},
  author={Shi, J. and Zhou, Y. and Hwang, J. and Titsias, M. and Mackey, L.},
  journal={Advances in Neural Information Processing Systems},
  volume={35},
  pages={25829--25841},
  year={2022}
}

@article{liu2018stein,
  title={Stein variational gradient descent as moment matching},
  author={Liu, Q. and Wang, D.},
  journal={Advances in Neural Information Processing Systems},
  volume={31},
  year={2018}
}

@article{dwivedi2024kernel,
  title={Kernel thinning},
  author={Dwivedi, R. and Mackey, L.},
  journal={Journal of Machine Learning Research},
  volume={25},
  number={152},
  pages={1--77},
  year={2024}
}

@article{moores2020scalable,
  title={Scalable {Bayesian} inference for the inverse temperature of a hidden Potts model},
  author={Moores, M. and Nicholls, G. K. and Pettitt, A. N. and Mengersen, K.},
  journal={Bayesian Analysis},
  volume={15},
  number={1},
  pages={1--27},
  year={2020},
  publisher={Carnegie Mellon University}
}

@book{silverman1972special,
  title={Special functions and their applications},
  author={Silverman, R. A. and Lebedev, N. N.},
  year={1972},
  publisher={Courier Corporation}
}

@article{bell1934exponential,
  title={Exponential polynomials},
  author={Bell, E. T.},
  journal={Annals of Mathematics},
  volume={35},
  number={2},
  pages={258--277},
  year={1934},
  publisher={JSTOR}
}

@book{comtet2012advanced,
  title={Advanced Combinatorics: The art of finite and infinite expansions},
  author={Comtet, L.},
  year={2012},
  publisher={Springer Science \& Business Media}
}

@article{touchard1939cycles,
  title={Sur les cycles des substitutions},
  author={Touchard, J.},
  journal={Acta Mathematica},
  year={1939}
}

@inproceedings{yang2018goodness,
  title={Goodness-of-fit testing for discrete distributions via {Stein} discrepancy},
  author={Yang, J. and Liu, Q. and Rao, V. and Neville, J.},
  booktitle={International Conference on Machine Learning},
  pages={5561--5570},
  year={2018},
  organization={PMLR}
}

@inproceedings{swersky2020amortized,
  title={Amortized {Bayesian} optimization over discrete spaces},
  author={Swersky, K. and Rubanova, Y. and Dohan, D. and Murphy, K.},
  booktitle={Conference on Uncertainty in Artificial Intelligence},
  pages={769--778},
  year={2020},
  organization={PMLR}
}

@article{daulton2022bayesian,
  title={Bayesian optimization over discrete and mixed spaces via probabilistic reparameterization},
  author={Daulton, S. and Wan, X. and Eriksson, D. and Balandat, M. and Osborne, M. A. and Bakshy, E.},
  journal={Advances in Neural Information Processing Systems},
  volume={35},
  pages={12760--12774},
  year={2022}
}

@inproceedings{gessner2020active,
  title={Active multi-information source {Bayesian} quadrature},
  author={Gessner, A. and Gonzalez, J. and Mahsereci, M.},
  booktitle={Uncertainty in Artificial Intelligence},
  pages={712--721},
  year={2020},
  organization={PMLR}
}

@article{gunter2014sampling,
  title={Sampling for inference in probabilistic models with fast {Bayesian} quadrature},
  author={Gunter, T. and Osborne, M. A. and Garnett, R. and Hennig, P. and Roberts, S. J.},
  journal={Advances in Neural Information Processing Systems},
  volume={27},
  year={2014}
}

@inproceedings{xi2018bayesian,
  title={Bayesian quadrature for multiple related integrals},
  author={Xi, X. and Briol, F.-X. and Girolami, M.},
  booktitle={International Conference on Machine Learning},
  pages={5373--5382},
  year={2018},
  organization={PMLR}
}

@inproceedings{paul2018alternating,
  title={Alternating optimisation and quadrature for robust control},
  author={Paul, S. and Chatzilygeroudis, K. and Ciosek, K. and Mouret, J.-B. and Osborne, M. and Whiteson, S.},
  booktitle={Proceedings of the AAAI Conference on Artificial Intelligence},
  volume={32},
  number={1},
  year={2018}
}

@article{kanagawa2020convergence,
  title={Convergence analysis of deterministic kernel-based quadrature rules in misspecified settings},
  author={Kanagawa, M. and Sriperumbudur, B. K. and Fukumizu, K.},
  journal={Foundations of Computational Mathematics},
  volume={20},
  number={1},
  pages={155--194},
  year={2020},
  publisher={Springer}
}

@article{osborne2012active,
  title={Active learning of model evidence using {Bayesian} quadrature},
  author={Osborne, M. and Garnett, R. and Ghahramani, Z. and Duvenaud, D. K. and Roberts, S. J. and Rasmussen, C.},
  journal={Advances in Neural Information Processing Systems},
  volume={25},
  year={2012}
}

@incollection{gautschi1981survey,
  title={A survey of {Gauss-Christoffel} quadrature formulae},
  author={Gautschi, W.},
  booktitle={EB Christoffel: The influence of his work on mathematics and the physical sciences},
  pages={72--147},
  year={1981},
  publisher={Springer}
}

@article{chen2025stationary,
  title={Stationary {MMD} Points for Cubature},
  author={Chen, Z. and Karvonen, T. and Kanagawa, H. and Briol, F.-X. and Oates, C. J. and others},
  journal={arXiv preprint arXiv:2505.20754},
  year={2025}
}

@inproceedings{chen2024conditional,
  title={Conditional {Bayesian} quadrature},
  author={Chen, Z. and Naslidnyk, M. and Gretton, A. and Briol, F.-X.},
  booktitle={Proceedings of the Fortieth Conference on Uncertainty in Artificial Intelligence},
  pages={648--684},
  year={2024}
}

@inproceedings{li2023multilevel,
  title={Multilevel {Bayesian} quadrature},
  author={Li, K. and Giles, D. and Karvonen, T. and Guillas, S. and Briol, F.-X.},
  booktitle={International Conference on Artificial Intelligence and Statistics},
  pages={1845--1868},
  year={2023},
  organization={PMLR}
}

@article{bardenet2020monte,
  title={{Monte Carlo} with determinantal point processes},
  author={Bardenet, R. and Hardy, A.},
  journal={The Annals of Applied Probability},
  year={2020}
}

@article{dick2011higher,
  title={HIGHER ORDER SCRAMBLED DIGITAL NETS ACHIEVE THE OPTIMAL RATE OF THE ROOT MEAN SQUARE ERROR FOR SMOOTH INTEGRANDS},
  author={Dick, J.},
  journal={The Annals of Statistics},
  pages={1372--1398},
  year={2011},
  publisher={JSTOR}
}

@article{chopin2024higher,
  title={Higher-order {Monte Carlo} through cubic stratification},
  author={Chopin, N. and Gerber, M.},
  journal={SIAM Journal on Numerical Analysis},
  volume={62},
  number={1},
  pages={229--247},
  year={2024},
  publisher={SIAM}
}

@book{novak2006deterministic,
  title={Deterministic and stochastic error bounds in numerical analysis},
  author={Novak, E.},
  volume={1349},
  year={2006},
  publisher={Springer}
}

@inproceedings{murray2006mcmc,
  title={{MCMC} for doubly-intractable distributions},
  author={Murray, I. and Ghahramani, Z. and MacKay, D. J. C.},
  booktitle={Proceedings of the 22nd Annual Conference on Uncertainty in Artificial Intelligence (UAI-06)},
  pages={359--366},
  year={2006},
  organization={AUAI Press}
}

@book{hennig2022probabilistic,
  title={Probabilistic Numerics: Computation as Machine Learning},
  author={Hennig, P. and Osborne, M. A. and Kersting, H. P.},
  year={2022},
  publisher={Cambridge University Press}
}

@article{amin2023biological,
  title={Biological sequence kernels with guaranteed flexibility},
  author={Amin, A. N. and Weinstein, E. N. and Marks, D. S.},
  journal={arXiv preprint arXiv:2304.03775},
  year={2023}
}

@article{jorgensen2015discrete,
  title={Discrete reproducing kernel {Hilbert} spaces: sampling and distribution of Dirac-masses},
  author={Jorgensen, P. and Tian, F.},
  journal={The Journal of Machine Learning Research},
  volume={16},
  number={1},
  pages={3079--3114},
  year={2015},
  publisher={JMLR. org}
}

@article{robins2007recent,
  title={Recent developments in exponential random graph (p*) models for social networks},
  author={Robins, G. and Snijders, T. and Wang, P. and Handcock, M. and Pattison, P.},
  journal={Social Networks},
  volume={29},
  number={2},
  pages={192--215},
  year={2007},
  publisher={Elsevier}
}

@book{moller2003statistical,
  title={Statistical Inference and Simulation for Spatial Point Processes},
  author={Moller, J. and Waagepetersen, R. P.},
  year={2003},
  publisher={CRC Press}
}

@article{balakrishnan2011learning,
  title={Learning generative models for protein fold families},
  author={Balakrishnan, S. and Kamisetty, H. and Carbonell, J. G. and Lee, S.-I. and Langmead, C. J.},
  journal={Proteins: Structure, Function, and Bioinformatics},
  volume={79},
  number={4},
  pages={1061--1078},
  year={2011},
  publisher={Wiley Online Library}
}

@article{briol2019probabilistic,
  title={Probabilistic integration: A role in statistical computation? (with discussion)},
  author={Briol, F.-X. and Oates, C. J. and Girolami, M. and Osborne, M. A. and Sejdinovic, D.},
  journal={Statistical Science},
  volume={34},
  number={1},
  pages={1--22},
  year={2019},
  publisher={JSTOR}
}

@inproceedings{chennested,
  title={Nested Expectations with Kernel quadrature},
  author={Chen, Z. and Naslidnyk, M. and Briol, F.-X.},
  booktitle={Forty-second International Conference on Machine Learning},
  year={2025}
}

@article{aronszajn1950theory, 
title={Theory of reproducing kernels}, 
author={Aronszajn, N.}, 
journal={Transactions of the American Mathematical Society}, 
volume={68}, number={3}, pages={337--404}, year={1950}}

\newpage

\begin{appendices}

\crefalias{section}{appendix}
\crefalias{subsection}{appendix}
\crefalias{subsubsection}{appendix}

\setcounter{equation}{0}
\renewcommand{\theequation}{\thesection.\arabic{equation}}
\newcommand{\appsection}[1]{
  \refstepcounter{section}
  \section*{Appendix \thesection: #1}
  \addcontentsline{toc}{section}{Appendix \thesection: #1}
}

\onecolumn

{\hrule height 1mm}
\vspace*{-0pt}
\section*{\LARGE\bf \centering Supplementary Material
}
\vspace{8pt}
{\hrule height 0.1mm}
\vspace{24pt}

\section{Detailed on Stein reproducing kernel over the discrete domain}\label{sec:details}
Let $\calX$ be a finite discrete domain with an ordering $x^{[1]}, x^{[2]}, \ldots, x^{[|\calX|]}$.
Define the cyclic permutation operator $\neg:\calX\to\calX$ and its inverse operator $\rneg:\calX\to\calX$ by $\neg x^{[i]}=x^{[(i+1) \bmod |\calX|]}$ and $\rneg x^{[i]}=x^{[(i-1) \bmod |\calX|]}$ for any $i=1,2, \ldots,|\calX|$.  It follows immediately that these operators are inverses of each other, i.e. $\neg(\rneg x)=\rneg(\neg x)=x$ for any $x \in \calX$. 

Given a cyclic permutation $\neg$ on $\calX$, for any vector $x=[x_1, \ldots, x_h]^{\top} \in \calX^h$, define $\neg_i$ as a cyclic permutation of the $i$-th location $\neg_i x:= [x_1, \ldots, x_{i-1}, \neg x_i, x_{i+1}, \ldots, x_h]^{\top}$. For any function $f: \calX^h \rightarrow \mathbb{R}$, define the partial difference operator as $\Delta_{x_i} f(x):=f(x)-f(\neg_i x), \quad i=1, \ldots, d$, and write $\Delta_x f(x)=[\Delta_{x_1} f(x), \ldots, \Delta_{x_h} f(x)]^{\top} \in\R^h$. 
The difference score function is defined as $s_p(x) := \frac{\Delta_x f(x)}{f(x)}$. 
Similarly, for the inverse cyclic permutation $\rneg$, 
define the associated partial difference operator as: $\Delta_{x_i}^\ast f(x):=f(x)-f\left(\rneg_i x\right), \quad i=1, \ldots, h$, 
and write $\Delta^\ast_x f(x) = [\Delta^\ast_{x_1} f(x), \ldots, \Delta^\ast_{x_h} f(x)]^{\top}\in\R^d$.
The superscript $\ast$ arises because $\Delta_x^\ast$ is the adjoint operator of $\Delta_x$ with respect to the standard $\ell_2$ inner product over $\calX^h$. 

\section{Derivations of Closed Form Kernel Mean Embeddings}\label{sec:derivation}
In this section, we provide detailed derivations of all the closed form kernel mean embeddings in \Cref{tab:discrete_kme}. 

\paragraph{Poisson distribution and Brownian motion kernel} 
Consider a Poisson distribution whose probability mass function $p(x) = \frac{\eta^x e^{-\eta}}{x!}$ for any $x\in\N_0$, and a Brownian motion kernel $k(x, y) = x\wedge y$. 
The kernel mean embedding can be computed as: for any $y\in\N_0$, 
\begin{align*}
\mu_\Pb(y) 
&= \sum_{x=0}^{\infty} \min(x, y) \cdot \frac{e^{-\eta} \eta^x}{x!} \\
&= \sum_{x=0}^{y} x \cdot \frac{e^{-\eta} \eta^x}{x!} 
+ \sum_{x=y+1}^{\infty} y \cdot \frac{e^{-\eta} \eta^x}{x!} \\
&= \sum_{x=0}^{y} x \cdot \frac{e^{-\eta} \eta^x}{x!} 
+ y \cdot \sum_{x=y+1}^{\infty} \frac{e^{-\eta} \eta^x}{x!} \\
&= \sum_{x=0}^{y} x \cdot \frac{e^{-\eta} \eta^x}{x!} 
+ y \cdot P(X > y) \\
&= \sum_{x=0}^{y} x \cdot \frac{e^{-\eta} \eta^x}{x!} 
+ y \cdot \left(1 - Q(y + 1, \eta) \right),
\end{align*}
where $Q$ is the regularized gamma function and it admits efficient computation in python \texttt{scipy} package. 

The initial error can be computed as:
\begin{align*}
\E_{X,X'\sim \Pb}[k(X,X')] &= \sum_{m=0}^{\infty} \sum_{n=0}^{\infty} \min(m, n) \cdot P(X = m) \cdot P(y = n) \\
&= \sum_{r=0}^{\infty} P(\min(X, y) > r) = \sum_{r=0}^{\infty} P(X > r \text{ and } y > r) \\
&= \sum_{r=0}^{\infty}  P(X > r)^2 = \sum_{r=0}^{\infty} \left( 1 - Q(r+1, \eta) \right)^2 ,
\end{align*}
which would involve an infinite sum over regularized gamma function and does not admit a closed form expression. In the experiments, we truncate the above infinite sum with first $100$ terms. 

\paragraph{Negative binomial distribution and Brownian motion kernel}
Consider a negative binomial distribution $\text{NB}(\tau,q)$ whose probability mass function $P(X = x) = \binom{x+\tau-1}{x}(1-q)^x q^\tau$ for any $x\in\N_0$, and a Brownian motion kernel $k(x, y) = x\wedge y$. 
The kernel mean embedding can be computed as: for any $y\in\N_0$, 
\begin{align*}
    \mu_\Pb(y) &= \sum_{x=0}^{\infty} \min(x, y) \cdot P(X = x) = \sum_{k=1}^y P(X \geq k)  = y - \sum_{k=1}^y P(X \leq k - 1) = y - \sum_{k=0}^{y-1} I_p(\tau,k+1) .
\end{align*}
Here, $I_p$ denotes regularized incomplete Beta function and it admits an efficient computation in python \texttt{scipy} package. 
The initial error can be computed as:
\begin{align*}
    \E_{X,X'\sim \Pb}[k(X,X')] &= \sum_{m=0}^{\infty}\left(1-I_p(\tau, m+1)\right)^2 .
\end{align*}
which unfortunately is another infinite sum that does not admit a closed form expression. 
In the experiments, we truncate the above infinite sum with the first $100$ terms. 

\paragraph{Logarithmic distribution and Brownian motion kernel}
Consider a logarithmic distribution $\text{Log}(q)$ whose probability mass function is given by $
P(X = x) = -\frac{1}{\ln(1-q)} \cdot \frac{q^x}{x}$ for any $x\in\N_+$, and a Brownian motion kernel $k(x, y) = \min(x,y)$. 
The kernel mean embedding can be computed as: for any $y \in \N$, 
\begin{align*}
    \mu_\Pb(y) &= \sum_{x=1}^{\infty} \min(x, y) \cdot P(X = x) 
    = \sum_{x=1}^y x \cdot P(X=x)+y \sum_{x=y+1}^{\infty} P(X=x) \\
    &= -\frac{1}{\ln (1-q)}\left(\sum_{x=1}^y q^x+y \sum_{x=y+1}^{\infty} \frac{q^x}{x}\right) = -\frac{1}{\ln (1-q)}\left(\sum_{x=1}^y q^x-y \ln (1-q)-y \sum_{x=1}^y \frac{q^x}{x}\right) \\
    &= y+\frac{1}{\ln (1-q)} \sum_{n=1}^{y-1} \frac{(y-n) q^n}{n} .
\end{align*}
Unfortunately, the initial error $\E_{X,X'\sim \Pb}[k(X,X')]$ does not admit a closed form expression.

\paragraph{Poisson distribution and polynomial kernel}
Consider a Poisson distribution whose probability mass function $p(x) = \frac{\eta^x e^{-\eta}}{x!}$ for any $x\in\N_0$, and a polynomial kernel $k(x, y) = (x y + 1)^r$. 
The kernel mean embedding can be computed as: for any $y\in\N_0$, 
\begin{align*}
    \mu_{\Pb}(y)=\sum_{n=0}^r \binom{r}{n} y^n \E[X^n] = \sum_{n=0}^r \binom{r}{n} y^n T_n(\eta) .
\end{align*}
where $T_n(\eta)$ is the Touchard polynomial~\citep{touchard1939cycles} and it admits an efficient computation via dynamical programming~\citep{comtet2012advanced}. The initial error can be computed as $\E_{X,X'\sim \Pb}[k(X,X')] = \sum_{n=0}^r \binom{r}{n} T_n(\eta)^2$.

\paragraph{Negative binomial distribution and polynomial kernel}
Consider a negative binomial distribution $\text{NB}(\tau,q)$ whose probability mass function $P(X=x) = \binom{x+\tau-1}{x}(1-q)^x q^\tau$ for any $x\in\N_0$, and a polynomial kernel $k(x, y) = (x y + 1)^r$. 
The kernel mean embedding can be computed as: for any $y\in\N_0$, 
\begin{align*}
    \mu_{\Pb}(y)=\sum_{n=0}^r \binom{r}{n} y^n \E[X^n] = \sum_{n=0}^r \binom{r}{n} y^n M_n(\tau,q) ,
\end{align*}
where $M_n(\tau, q)=\sum_{j=0}^n S(n, j)j!\binom{\tau+j-1}{j}(\frac{1-q}{q})^j$ with $S(n,j)$ being the Stirling numbers~\citep{comtet2012advanced}. 
The initial error can be computed as $\E_{X,X'\sim \Pb}[k(X,X')] = \sum_{n=0}^r \binom{r}{n} M_n(\tau,q)^2$. 

\paragraph{Skellam distribution and polynomial kernel}
Consider a Skellam distribution whose probability mass function $P(X=x) = e^{-(\lambda_1+\lambda_2)}(\frac{\lambda_1}{\lambda_2})^{x / 2} I_{|x|}(2 \sqrt{\lambda_1 \lambda_2}) $ for any $x\in\N_0$, and a polynomial kernel $k(x, y) = (x y + 1)^r$. 
The kernel mean embedding can be computed as: for any $y\in\N_0$, 
\begin{align*}
    \mu_{\Pb}(y)=\sum_{n=0}^r \binom{r}{n} y^n \E[X^n] = \sum_{n=0}^r \binom{r}{n} y^n B_n(\lambda_1 - \lambda_2, \ldots, \lambda_1 + (-1)^n \lambda_2) ,
\end{align*}
where $B_n$ are the complete exponential Bell polynomials~\citep{bell1934exponential}.  
The initial error can be computed as $\E_{X,X'\sim \Pb}[k(X,X')] = \sum_{n=0}^r \binom{r}{n} B_n(\lambda_1 - \lambda_2, \ldots, \lambda_1 + (-1)^n \lambda_2)^2$. 

\paragraph{Uniform distribution over $\{-1, +1\}^d$ and exponential Hamming distance kernel}
Consider a uniform distribution over all possible states $\{-1, +1\}^d$ of an Ising model, and an exponential Hamming distance kernel $k(x, y) = \exp(- \lambda d_H(x,y))$ with a lengthscale $\lambda > 0$. 
In this particular domain of $\{-1, +1\}^d$, the exponential Hamming distance kernel can be equivalently written as $k(x, y) = \exp(- \frac{\lambda}{2}(d - \sum_{i=1}^d x_i y_i))$. As a result, the kernel mean embedding can be computed as follows:
\begin{align*}
    \mu_\Pb(y) &= \frac{1}{2^d} \sum_{x\in\{-1, +1\}^d} \left[ \exp\left(-\frac{\lambda d}{2}\right) \cdot \exp\left( \frac{\lambda}{2} \sum_{i=1}^d x_i y_i \right) \right] \\
    &= \frac{1}{2^d} \cdot \exp\left(-\frac{\lambda d}{2}\right) \sum_{x\in\{-1, +1\}^d} \prod_{j=1}^d \exp\left( \frac{\lambda}{2} x_j y_j \right) \\
    &= \frac{1}{2^d} \cdot \exp\left(-\frac{\lambda d}{2}\right) \prod_{j=1}^d \sum_{x_j \in \{-1,1\}} \exp\left( \frac{\lambda}{2} x_j y_j \right) \\
    &= \exp\left(-\frac{\lambda d}{2}\right) \prod_{j=1}^d \cosh\left( \frac{\lambda y_j}{2} \right) \\
    &= \exp\left(-\frac{\lambda d}{2}\right) \left[ \cosh\left( \frac{\lambda}{2} \right) \right]^d . 
\end{align*}
Since the kernel mean embedding is a constant function, the initial error is the same as the kernel mean embedding, i.e. $\E_{X,X'\sim\Pb}[k(X,X')] = \exp(-\frac{\lambda d}{2}) [ \cosh(\frac{\lambda}{2})]^d$.

\paragraph{Uniform distribution over $\{0, +1\}^d$ and Tanimoto kernel} 
Consider a uniform distribution over all possible states $\{0, +1\}^d$ of an Ising model, and a Tanimoto kernel $k(x, y) = \frac{x^\top y}{\|x\|^2+\|y\|^2-x^\top y}$. The kernel mean embedding can be computed as follows:
\begin{align}
    \mu_\Pb(y) &= \frac{1}{2^d} \sum_{x\in\{0, +1\}^d} \frac{x^\top y}{\|x\|^2 + \|y\|^2 - x^\top y} \nonumber \\
    &= \frac{1}{2^d} \sum_{s=0}^d \sum_{\substack{x: \|x\|^2 = s}} \frac{x^\top y}{s + \|y\|^2 - x^\top y} \nonumber \\
    &= \frac{1}{2^d} \sum_{s=0}^d \sum_{a = \max(0, s - d + \|y\|^2)}^{\min(s, \|y\|^2)} \left[ \sum_{\substack{\|x\|^2 = s \\ x^\top y = a}} \frac{a}{s + \|y\|^2 - a} \right] \nonumber \\
    &= \frac{1}{2^d} \sum_{s=0}^d \sum_{a = \max(0, s - d + \|y\|^2)}^{\min(s, \|y\|^2)} \left[ \frac{a}{s + \|y\|^2 - a} \cdot \binom{\|y\|^2}{a} \cdot \binom{d - \|y\|^2}{s - a} \right] \label{eq:kme_tanimoto}. 
\end{align}
Then, the initial error would be
\begin{align*}
    \E_{X,X'\sim\Pb}[k(X,X')] = \frac{1}{2^d} \sum_{y\in\{0, +1\}^d}\mu_\Pb(y) = \frac{1}{2^{2d}} \sum_{t=0}^d \sum_{s=0}^d \sum_{a = \max(0, s - d + t)}^{\min(s, t)} \left[ \frac{a}{s + t - a} \cdot \binom{d}{t} \cdot \binom{t}{a} \cdot \binom{d - t}{s - a} \right] . 
\end{align*}
The total number of terms to sum over is $\calO(d)$ for the kernel mean embedding and $\calO(d^2)$ for the initial error, both of which are substantially cheaper than the brutal force summation of $\calO(2^d)$ terms. 

\paragraph{Uniform distribution over $\{0, 1, 2\}^d$ and exponential Hamming distance kernel} 
Consider a uniform distribution over all possible states $\{0, 1,2 \}^d$ of a Potts model, and an exponential Hamming distance kernel kernel $k(x, y) = \exp(- \lambda d_H(x,y))$ with a lengthscale $\lambda > 0$.  The kernel mean embedding can be computed as follows: 
\begin{align*}
    \mu_\Pb(y)= \frac{1}{3^d} \sum_{x} \exp(-\lambda d_H(x, y)) = \frac{1}{3^d} \sum_{s=0}^d \binom{d}{s} \cdot 2^s \cdot e^{-\lambda s} = \frac{1}{3^d} \sum_{s=0}^d \binom{d}{s} (2e^{-\lambda})^s = \left( \frac{1 + 2e^{-\lambda}}{3}\right)^d . 
\end{align*}
Since the kernel mean embedding is a constant over $y$, the initial error is the same as the kernel mean embedding, i.e. $\E_{X,X'\sim\Pb}[k(X,X')] = ( \frac{1 + 2e^{-\lambda}}{3})^d$. The same derivations can be extended to uniform distribution over $\{0, 1, \ldots, K\}^d$ for any positive integer $K$ with the corresponding kernel mean embedding $\mu_\Pb(y) = ( \frac{1 + Ke^{-\lambda}}{K+1})^d$ and initial error $\E_{X,X'\sim\Pb}[k(X,X')] = ( \frac{1 + Ke^{-\lambda}}{K+1})^d$.

\section{An Example on the Tightness of the Upper Bound in \Cref{thm:main}}\label{sec:example}
The worst-case optimality results cited in the proof of Theorem~\ref{thm:main} state that
\begin{equation} \label{eq:wce-definition}
    \hat{\sigma} = \sup_{\lVert f \rVert_{\mathcal{H}_k} \leq 1} \bigg\lvert \sum_{x \in \mathcal{X}} f(x) p(x) - \sum_{i=1}^N  w_i f(x_i) \bigg\rvert,
\end{equation}
where $w_i$ are the BayesSum weights (see Remark~\ref{rmk:weights}).
Suppose that $\mathcal{X} = \mathbb{N}$ and $k$ is any Matérn kernel.
The RKHS of $k$ on $\mathbb{R}$ is a Sobolev space~\cite{kanagawa2018gaussian}.
Being compactly supported and infinitely differentiable, the bump function $f$ given by $f(x) = \exp(-1/(1-x^2))$ for $\lvert x \vert < 1$ and $f(x) = 0$ otherwise is in every Sobolev space.
Because Matérn kernels are stationary, the norm of the translation $f(\cdot - a)$ does not depend on $a$.
Since $f(\cdot - a)|_{\mathbb{N}} \in \mathcal{H}_k$ for every $a$, by selecting $a = N+1$ and using~\eqref{eq:wce-definition} we obtain
\begin{equation*}
    \hat{\sigma} \geq \frac{1}{\lVert f|_{\mathbb{N}} \rVert_{\mathcal{H}_k}} f(0) p(N+1) = \frac{\exp(-1)}{\lVert f|_{\mathbb{N}} \rVert_{\mathcal{H}_k}} p(N+1) ,
\end{equation*}
where we used the fact that $f(i - (N+1)) = 0$ if $i \neq N+1$.
In combination with Theorem~\ref{thm:main} we thus have
\begin{equation*}
    C_1 p(N+1) \leq \hat{\sigma} \leq C_2 \sum_{i > N} p(i)
\end{equation*}
for all $N \in \mathbb{N}$ and some constants $C_1$ and $C_2$.
If $p(i) = \Theta(i^{-\alpha})$ for $\alpha > 1$, the lower bound is of order $N^{-\alpha}$ and the upper bound of order $N^{-\alpha+1}$.
If $p(i) = \Theta(e^{-ci})$ for $c > 0$, both the lower and upper bound are of order $e^{-cN}$.

\section{Additional experiments and details}\label{sec:additional_experiments}
\begin{figure*}[t]
    \centering
    \begin{subfigure}{0.33\textwidth}
        \centering
        \hspace{-10pt}
        \includegraphics[width=1.0\linewidth]{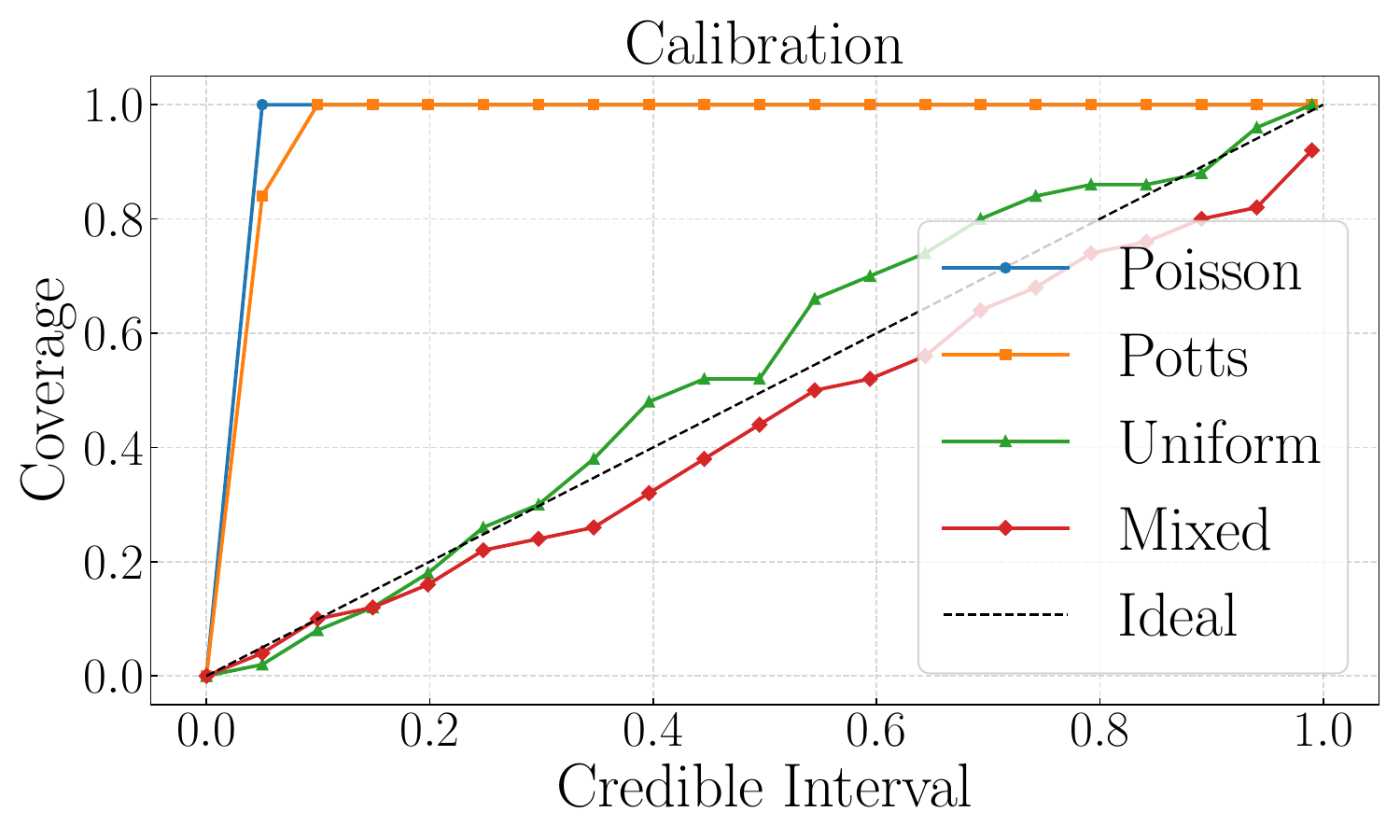}
    \end{subfigure}%
    \hfill 
    \begin{subfigure}{0.33\textwidth}
        \centering
        \hspace{-10pt}
        \includegraphics[width=1.0\linewidth]{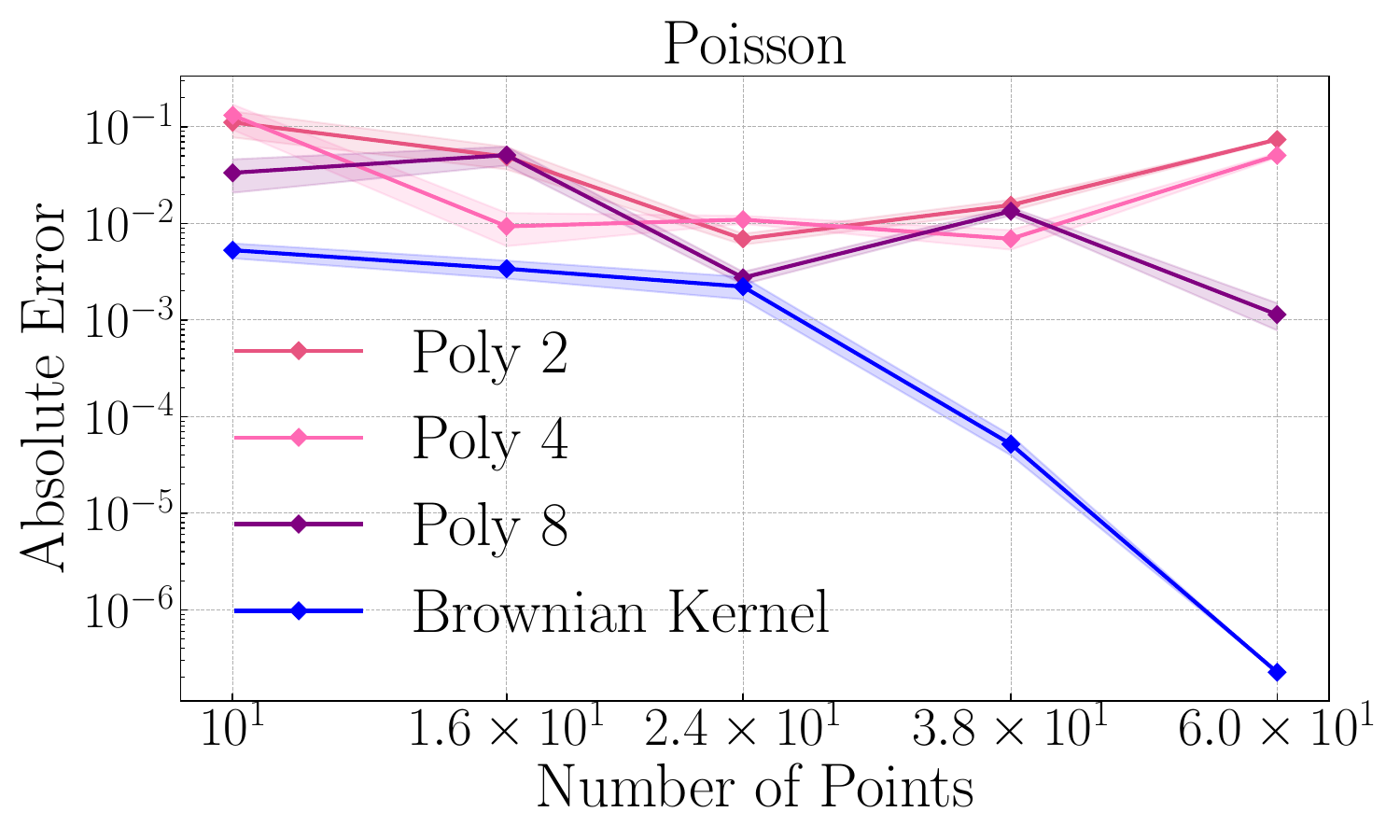}
    \end{subfigure}%
    \hfill 
    \begin{subfigure}{0.33\textwidth}
        \centering
        \hspace{-10pt}
        \includegraphics[width=1.0\linewidth]{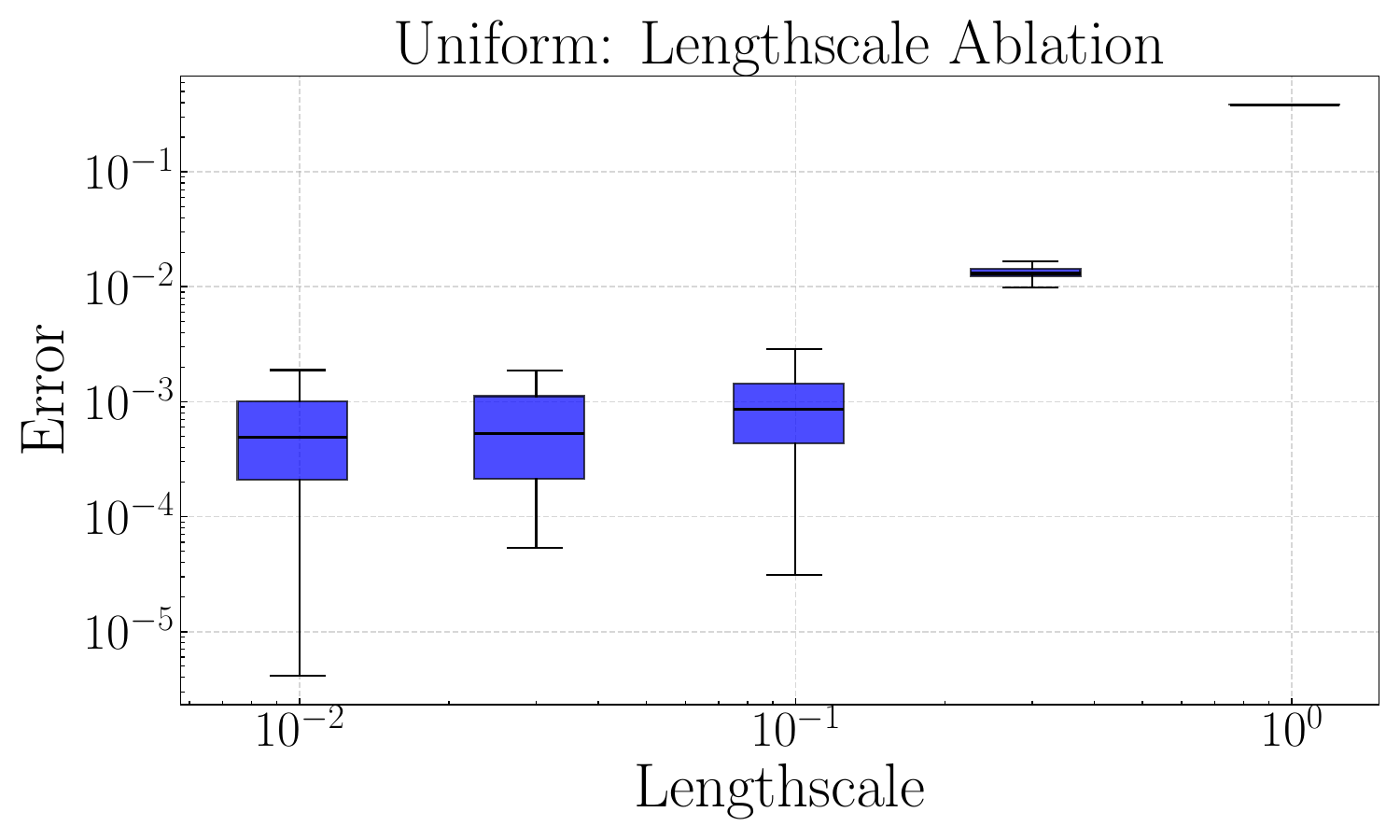}
    \end{subfigure}
    \vspace{-3pt}
    \caption{\textbf{Left}: Calibration of the posterior variance of BayesSum in synthetic settings (a)---(d). \textbf{Middle}: Ablation study on different kernels in synthetic setting (a): polynomial kernels of degree 2, 5, and 8 and Brownian motion kernel. 
    \textbf{Right}: Ablation study of kernel lengthscales of the the exponential Hamming kernel in synthetic setting (b). }
    \label{fig:toy_additional}
\end{figure*}

\textbf{More details on the synthetic settings: } 
Here, we provide details on the baseline methods Monte Carlo (MC), Russian Roulette (RR), importance sampling (IS), and stratified sampling (SS).
\begin{enumerate}[label=(\alph*), itemsep=2.0pt, topsep=2pt, leftmargin=*]
    \item The ground truth is computed by summation over the first 200 terms $n = 0, ... , 200$. For IS we use a negative binomial distribution $\mathrm{NB}(\tau,p)$ as the proposal distribution $\Qb$. 
    We pick $\tau = 5$ and the success probability $p=\eta / (\tau + \eta)$ such that the mean of the proposal distribution matches the mean of the distribution $\Pb=\text{Poisson}(\eta)$. 
    For RR, we draw a random subset $\mathcal{X}^{\prime} \subset \mathcal{X}$ according to a geometric distribution with parameter $\rho_{\star}=0.95$. Hence, the inclusion probability for the element at position $x'\in\calX$ is given by $\rho_\ast^j$. 
    For SS, the set $\calX=\N$ is partitioned into two disjoint regions, $\mathcal{R}_1=\{n \in \mathbb{N}: n<c\}$ and $\mathcal{R}_2=\{n \in \mathbb{N}: n \geq c\}$. The cutoff parameter \(c = 100\). An equal number of $N/2$ samples are drawn independently via inverse cumulative distribution function of each region. 
    
    \item The ground truth is computed by exact summation over the full index set. For IS, we uses a fully factorized proposal distribution $\Qb$ with \( q(x) = \prod_{i=1}^d q_i(x_i) \), where each \(q_i(x_i) \propto \exp(-\beta\, h_i^\top x_i)\) is a categorical distribution. 
    For RR, we first denote $\calI = \{0, 1, \ldots, 3^L\}$ as the index set of all possible sequences in $\calX=\{0, 1,2\}^L$. Hence, each index $j\in\calI$ has a bijective correspondence between the sequence in $\calX$. Next, we draw the random index subset $\mathcal{I}^{\prime} \subset \mathcal{I}$ according to a geometric distribution with parameter $\rho_*=0.95$. Therefore, for each index $j\in\calI$, the inclusion probability is given explicitly by $\rho_\ast^j$. 
    For SS, we partition the index set $\calI$ into four index subsets $\mathcal{R}_1, \mathcal{R}_2, \mathcal{R}_3, \mathcal{R}_4$ of equal sizes. Then we draw $\lfloor N/4 \rfloor$ number of samples uniformly and independently from each subset, which gives $N$ samples in total. 

    \item The ground truth is computed by exact enumeration over all configurations. Samples from the un-normalized Potts model are drawn using either Gibbs sampling or the Metropolis–Hastings algorithm~\citep{gelman1995bayesian}, with a burn-in of 10 iterations and thinning by a factor of 5. 
    Both MC and BayesSum are computed with the same set of samples in order for consistent comparison. 
    
    \item The ground truth admits a closed form expression: 
   \[I=\frac{1}{17^{2}}\sum_{h_{1}\in H}\sum_{h_{2}\in H}\Bigg[-20\,\frac{1 -e^{-0.2\,\alpha L}}{0.2\,\alpha L}\;-\;\Bigg(I_{0}(1)+\frac{2}{\beta L}\sum_{n=1}^{\infty} \frac{I_{n}(1)}{n}\,\sin(n\beta L)\Bigg)+ 20 + e \Bigg],\] \[
   \alpha = 1 + 0.1\,(h_{1} + h_{2}), \qquad \beta = 2\pi\!\left(1 + 0.05\,(h_{1}^{2} + h_{2}^{2})\right), \qquad H = \{\, -1 + 0.125\,j \;:\; j = 0,1,\ldots,16 \,\}\]
   and \(I_n(x)\) denotes the modified Bessel function of the first kind of order.

    Moreover, in this setting, we use the following kernel 
    \begin{align*}
        k\big((x,y),(x',y')\big) = k_\calX(x,x') + k_\calY(y,y') + k_\calX(x,x') \, k_\calY(y,y'), 
    \end{align*}
    where $k_\calX$ depends only on the continuous coordinates and $k_\calY$ depends only on the categorical coordinates.
    In particular, \(k_\calX(x,x')\) is the Gaussian kernel with lengthscale $\ell = 0.8$; and \(k_\calY(y,y')\) is the exponential--Hamming kernel with lengthscale $\lambda = 1.0$. 
    Under the joint distribution $\Pb = \Pb_X \times \Pb_Y$, The kernel mean embedding at $(x',y')$ equals $\mu_\Pb(x',y') = \mu_{\Pb_X}(x') + \mu_{\Pb_Y}(y') + \mu_{\Pb_X}(x') \cdot \mu_{\Pb_Y}(y')$. Here, both kernel mean embeddings $\mu_{\Pb_X}$ and $\mu_{\Pb_Y}$ admit closed form expressions from \citet{briol2025dictionary} and our \Cref{tab:discrete_kme}. 
\end{enumerate}

\textbf{Additional experimental results:}  
In \Cref{fig:toy_additional} (\textbf{Left}), we report the calibration of the BayesSum posterior variance. Here, coverage refers to the proportion of times that a confidence interval contains the true value across repeated experiments. 
The black diagonal line denotes perfect calibration, while curves above or below this line indicate underconfidence or overconfidence, respectively. 
We observe that the posterior variances are well-calibrated for the uniform distribution over $\{0,1,2\}^d$ in (b) and the distribution over mixed domain in (d), whereas the posterior variance is underconfident for the Poisson distribution in (a) and the Potts distribution in (c). 
This behavior arises because the Brownian motion kernel $k(x, y) = A(\min(x,y))$ in (a) and the Stein kernel in (c) are both unbounded, which leads to inflated posterior variances even when the kernel amplitude $A$ is carefully selected. 

In \Cref{fig:toy_additional} (\textbf{Middle}), we run an ablation study on the choice of kernels in case (a) of the synthetic setting using non-repetitive samples. The Brownian motion kernel achieves lower error compared to polynomial kernels since the integrand function is not an polynomial function. 
In the mean time, the performance of polynomial kernels improve as the degree of the polynomial increases from $2$ to $8$. 
This is expected since the higher-degree polynomial kernels can better approximate the integrand function.
In \Cref{fig:toy_additional} (\textbf{Right}), we run an ablation study on the choice of lengthscales in the exponential Hamming distance in case (b) of the synthetic setting. 
The performance are similar for lengthscales $0.01, 0.03, 0.1$ whereas larger lengthscales $0.3, 1.0$ achieve worse performances. This highlights the importance of lengthscale selection in BayesSum. 
In the experiments in the main, we optimize the kernel lengthscales via maximizing the log marginal likelihood, which gives a lengthscale of $0.01$, falling within the optimal range as shown in the ablation study as desired. 

In \Cref{fig:iid_repetitive}, we compare the absolute integration error of BayesSum under two sampling schemes:
(i) points drawn i.i.d. from a distribution with replacement; and
(ii) points drawn i.i.d. from the same distribution without replacement.
This distinction is unique to discrete domains, since repeated samples occur with probability zero in continuous spaces.
As illustrated in \Cref{fig:iid_repetitive}(\textbf{Left}) and (\textbf{Right}), using unique (non-duplicated) points consistently yields lower error. This is expected: repeated evaluation points provide no additional information to the BayesSum estimator and therefore do not improve its accuracy.
In \Cref{fig:iid_repetitive}(\textbf{Middle}), the two curves are nearly identical. This is because the state space has cardinality $3^{15}$
; when drawing samples i.i.d. from the uniform distribution over 
$\calX=\{0,1,2\}^{15}$, duplicates almost never occur in practice. Consequently, the i.i.d. and unique-sample schemes achieve indistinguishable performance. 

\textbf{More details on the unnormalised models:} 
Here, we provide additional details for the two unnormalised models: (a) the Conway–Maxwell–Poisson (CMP) model for count data, and (b) the Potts model for sequence data.
\begin{enumerate}[label=(\alph*), itemsep=2.0pt, topsep=2pt, leftmargin=*]
    \item The initial parameters are set to $(\theta_1, \theta_2) = (0.5, 1.2)$. BayesSum employs a Brownian motion kernel with amplitude~$1$. In \Cref{fig:trajectory}, the optimal parameter location is obtained by computing the log-likelihood over the entire grid $[0.0, 3.0] \times (0.0, 0.8]$ and selecting the point that achieves the maximum value.

    \item The initial parameters $\boldsymbol{\theta} = (\mathbf{h}, \mathbf{J})$ are drawn independently from a normal distribution $\mathcal{N}(0,\,0.01^2)$ at each coordinate. BayesSum employs an exponential Hamming kernel \(k(x,y) = \exp\bigl(-\lambda\, H(x,y)\bigr),\) with lengthscale $\lambda = 0.1104$ and amplitude $1$. 
    The discrete KSD values in \Cref{fig:potts_protein} is based on the discrete difference Stein operator of \citet{yang2018goodness}: the cyclic permutation on the single-site state space $\{0,1,2\}$ is defined as $s \mapsto (s+1) \bmod 3$ (i.e.\ $0 \to 1 \to 2 \to 0$), and is applied to the one-hot-encoded vectors; its inverse is implemented by a cyclic shift in the opposite direction.
    The discrete KSD values are computed between the trained Potts model $p(x;\boldsymbol{\theta})$ and a fixed set of $2000$ samples drawn from the true data distribution. 
    For KSD computation, We use the same exponential Hamming kernel as in BayesSum.
    
\end{enumerate}

\begin{figure*}[t]
    \centering
    \begin{subfigure}{0.33\textwidth}
        \centering
        \hspace{-10pt}
        \includegraphics[width=1.0\linewidth]{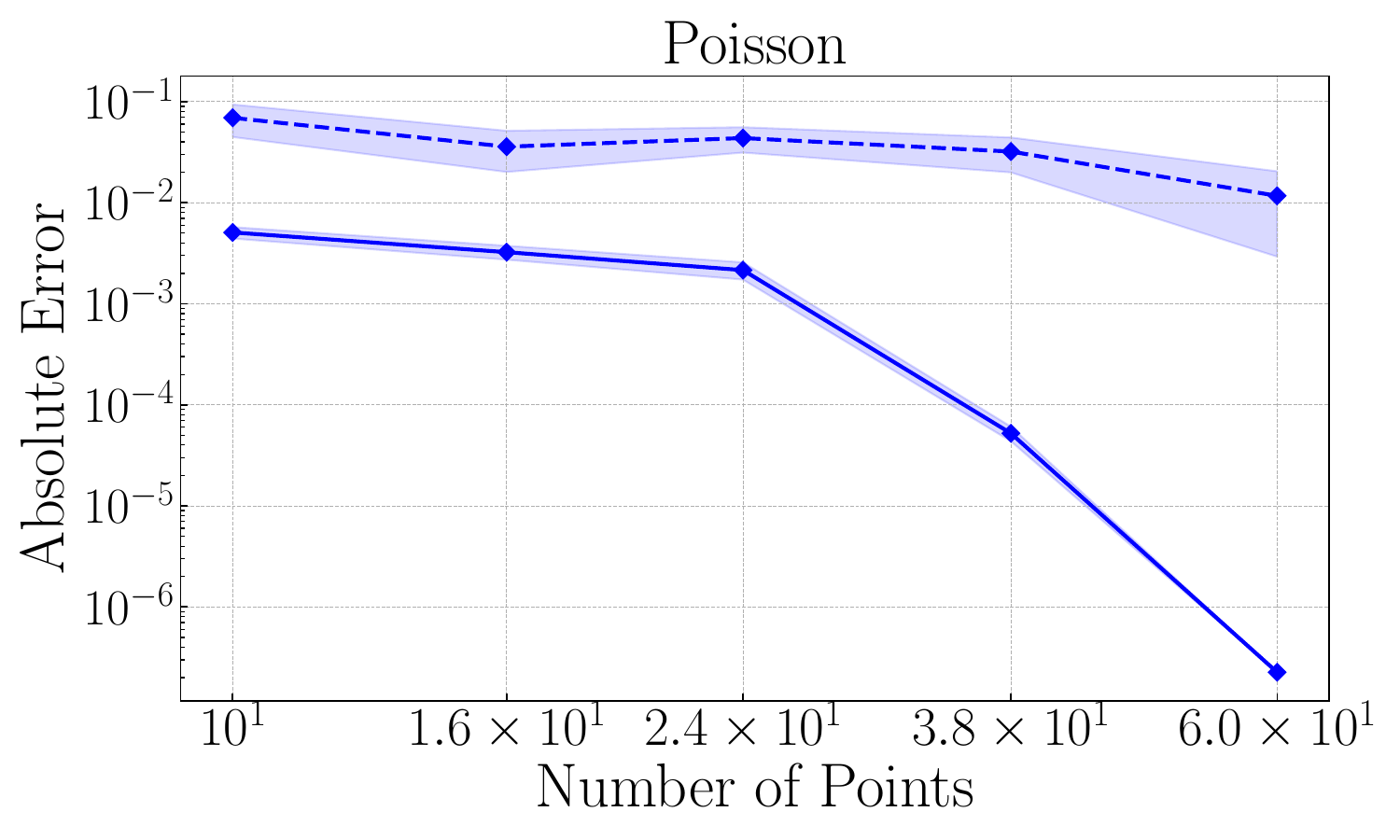}
    \end{subfigure}%
    \hfill 
    \begin{subfigure}{0.33\textwidth}
        \centering
        \hspace{-10pt}
        \includegraphics[width=1.0\linewidth]{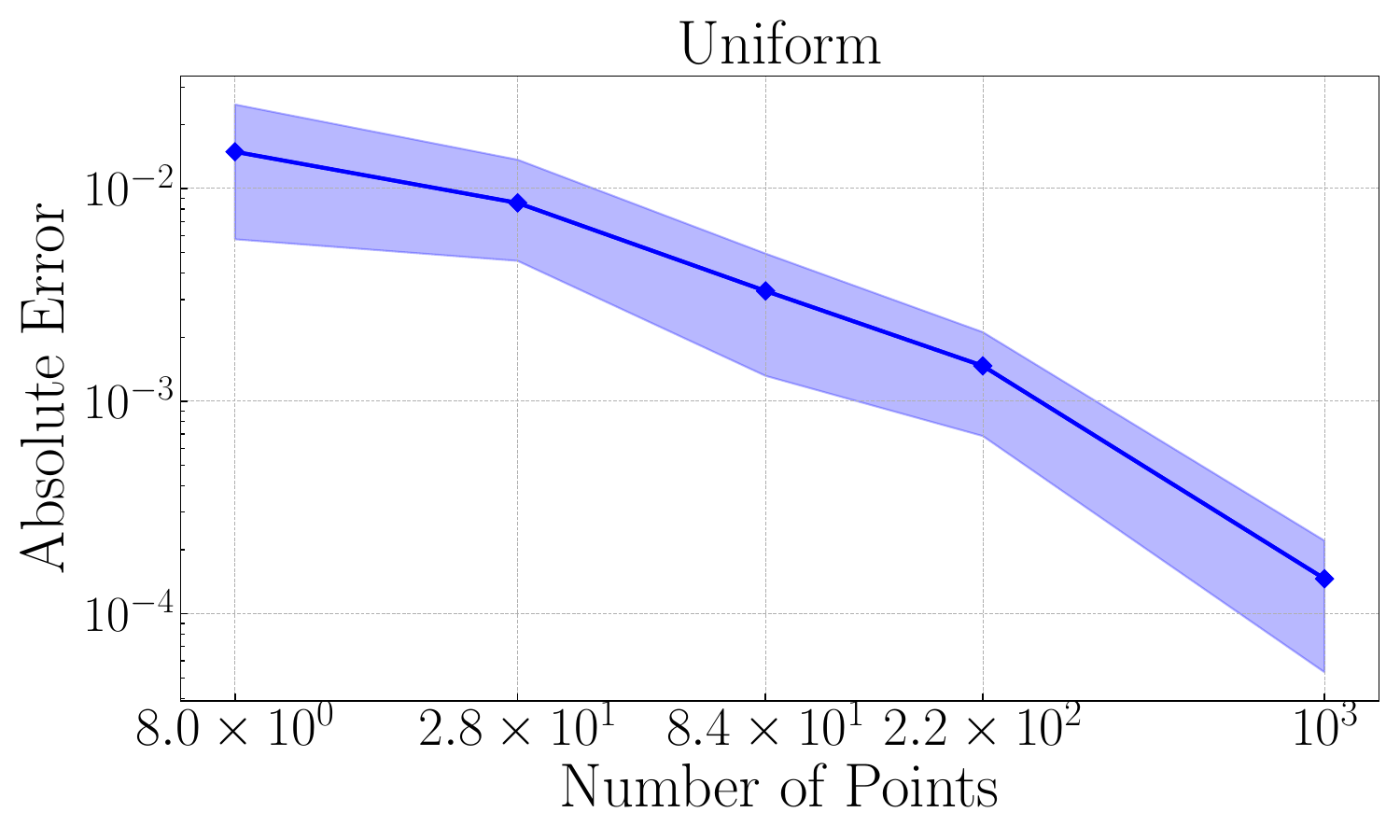}
    \end{subfigure}%
    \hfill 
    \begin{subfigure}{0.33\textwidth}
        \centering
        \hspace{-10pt}
        \includegraphics[width=1.0\linewidth]{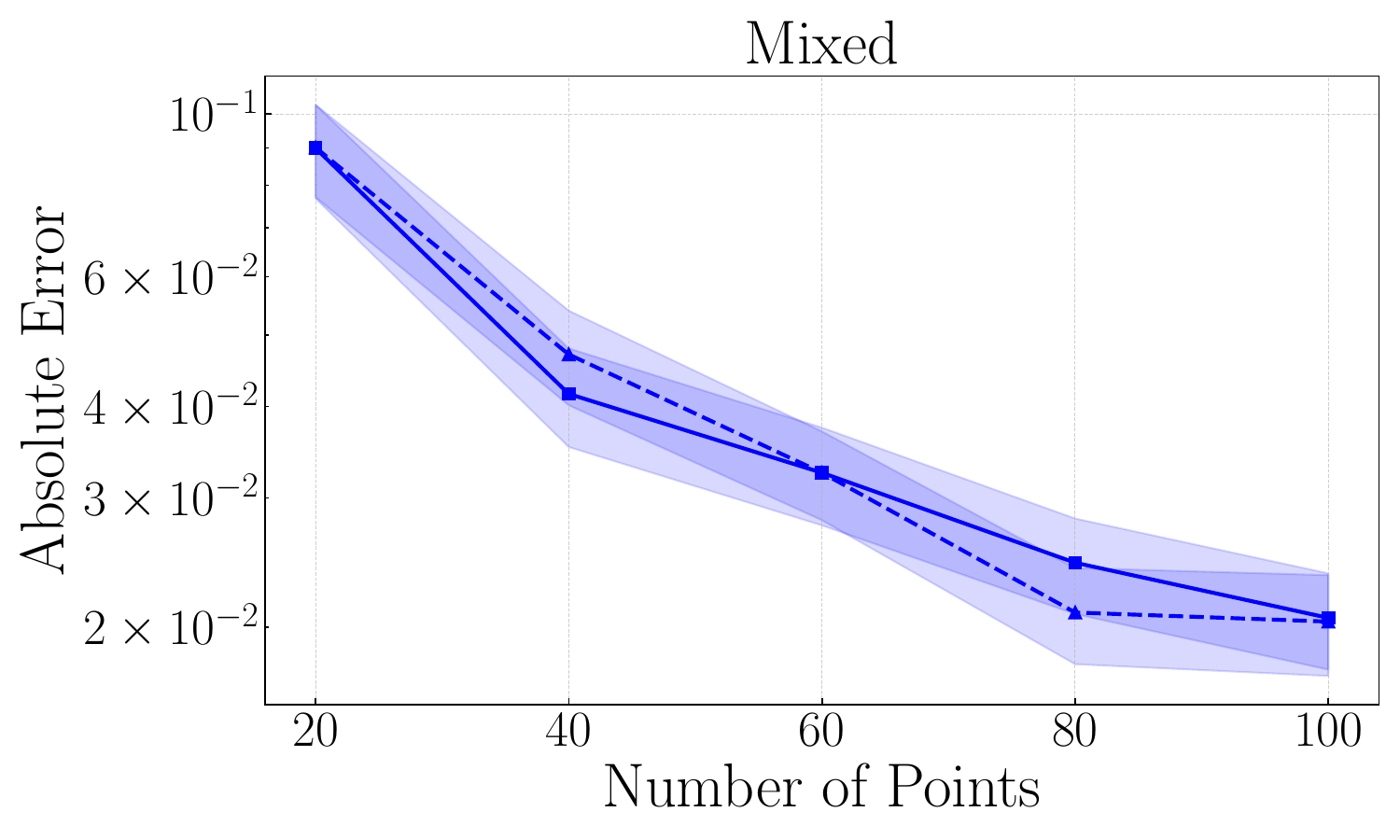}
    \end{subfigure}
    \vspace{-3pt}
    \caption{Comparison of BayesSum with i.i.d samples with replacement (real line) and without replacement (dashed line) in the setting of the Poisson distribution (\textbf{Left}), Uniform distribution (\textbf{Middle}) and the mixed distributions over continuous and discrete domain (\textbf{Right}). }
    \label{fig:iid_repetitive}
\end{figure*}

\end{appendices}

\newpage
\end{document}